\documentclass[3p,times]{elsarticle}
\usepackage{lineno}
\modulolinenumbers[5]
\usepackage[utf8]{inputenc} % allow utf-8 input
\usepackage[T1]{fontenc}    % use 8-bit T1 fonts
\usepackage{hyperref}       % hyperlinks
\usepackage{url}            % simple URL typesetting
\usepackage{booktabs}       % professional-quality tables
\usepackage{amsfonts}       % blackboard math symbols
\usepackage{nicefrac}       % compact symbols for 1/2, etc.
\usepackage{microtype}      % microtypography
\usepackage{xcolor}         % colors
\usepackage{algorithm}
\usepackage{algpseudocode}
\usepackage{amsmath}
\usepackage{amsthm}
\usepackage{amssymb}
\usepackage{multirow}
\usepackage{graphicx}
\usepackage{subfig}

\usepackage{subcaption}
\usepackage{adjustbox}
\usepackage{mathrsfs}
\algtext*{EndIf}
\algtext*{EndFor}
\algtext*{EndProcedure}
\newtheorem{prop}{Proposition}
\newcommand{\bogus}[1]{}

\biboptions{authoryear}

\begin{document}
\begin{frontmatter}

\title{GeoQ: Geometry-Aware Conditional Quantile Error Estimation for Scientific Surrogate Models}

\author[lanl]{Khoa Nguyen\fnref{equalcontrib}}
% \corref{mycorrespondingauthor}\cortext[mycorrespondingauthor]{Corresponding Author}
% \ead{nguyen_k@lanl.gov}
\author[lanl]{Daniel Serino\fnref{equalcontrib}}
\author[lanl]{Aviral Prakash}
\author[lanl]{Marc Klasky}

\fntext[equalcontrib]{These authors contributed equally to this work.}

\address[lanl]{Los Alamos National Laboratory, Los Alamos, NM 87545, USA}

\date{}

\begin{abstract}
Neural-network surrogate models are increasingly used to accelerate scientific simulations, but their deployment in extrapolative and autoregressive settings requires input-dependent estimates of prediction error. In this work, we introduce GeoQ (Geometry-Aware Conditional Quantile Error Estimation), a non-intrusive calibration framework for estimating surrogate error at individual query points. GeoQ represents the error at a query point as an anchor-averaged calibration error plus a learned nonnegative correction. This correction is modeled as an upper conditional quantile of the anchor-relative error increment, using geometry-based features that encode representation-space displacement and local support density. A cross-fitting procedure generates approximately out-of-sample calibration tuples, while a feature-space $k$-nearest-neighbor support score identifies regions \textcolor{black}{where the learned error model is supported by calibration data}. We evaluate GeoQ on scalar regression, chaotic dynamics, medium-range weather forecasting, and Richtmyer-Meshkov instability prediction. The results demonstrate that geometry-aware conditional quantile modeling provides a practical and non-intrusive approach for validity-aware error estimation in scientific surrogate models.
\end{abstract}


\end{frontmatter}

\section{Introduction}

Neural-network surrogate models are increasingly used to approximate expensive scientific simulations and solution operators arising in fluid dynamics, weather prediction, and operator-learning problems that map parameters, initial conditions, or forcing terms to quantities of interest~\citep{li2021fourier,kovachki2023neural,pathak2022fourcastnet}. These models can provide substantial computational speedups and enable parameter studies, inverse problems, design optimization, and long-time prediction that would be difficult with repeated high-fidelity simulation. In many scientific workflows, however, the surrogate is embedded in a larger decision or prediction loop. For instance, it may be queried within an optimizer, evaluated across a design space, used inside an inverse problem, or advanced autoregressively over many time steps. 
In such settings, it is important to understand how prediction error depends
on the specific query, since localized errors can influence downstream
optimization, inference, and long-horizon prediction. This motivates
query-dependent error estimators that assess surrogate accuracy at individual
inputs.
%In such settings, the relevant question is how prediction error varies as a function of the query point in the input domain. This motivates query-dependent error estimators that report surrogate accuracy as a function of the query.

Standard validation workflows give only a partial answer to this problem. Validation loss, test error, cross-validation, and early stopping are essential tools for model selection and overfitting control, but they mainly summarize average performance over a chosen validation or test distribution~\citep{stone1974cross,hastie2009elements,mohri2018foundations}. Classical generalization theory gives a complementary view by relating empirical risk, model complexity, and expected prediction error, but its guarantees are typically distribution-level statements. These tools are foundational, but they do not directly provide an input-dependent estimate of surrogate error, especially near the boundary of the training support or beyond it. This gap is especially important for modern overparameterized neural networks, where tight and practically useful generalization bounds remain challenging~\citep{zhang2017rethinking,dziugaite2017nonvacuous}.

In this work, we introduce \textit{GeoQ} (\textit{Geo}metry-Aware Conditional \textit{Q}uantile Error Estimation), a non-intrusive framework for estimating surrogate error at individual query points using calibration data with known reference outputs. The construction is motivated by a simple regularity principle: if the reference map and surrogate vary in a controlled way over a region of interest, then the error at a query should be related to errors at geometrically nearby reference points, which we call \emph{anchors}. Because a global worst-case bound based on Lipschitz continuity is often overly conservative, GeoQ uses the anchor errors to define a local baseline and learns a nonnegative correction through quantile regression~\citep{koenker1978regression,koenker2001quantile}. At a high level, for a query point $x_q\in\mathcal X$, GeoQ estimates the surrogate error $E(x_q)$ as
\begin{equation}
\widehat E(x_q)
=
\overline E_{\mathcal A}(x_q)
+
g_\eta\bigl(\psi(x_q,\mathcal A)\bigr),
\label{eq:geoq_intro_estimator}
\end{equation}
where $\mathcal A$ denotes a set of nearby calibration anchors, $\overline E_{\mathcal A}(x_q)$ is the average of their observed surrogate errors, and $g_\eta$ estimates an upper conditional quantile of the nonnegative anchor-relative error increment.  \textcolor{black}{Depending on the application, $E(x_q)$ may denote a scalar error metric, a vector of componentwise errors, or a pointwise error field.} 

The geometric features $\psi(x_q,\mathcal A)$ describe the query's
displacement from its anchors and its local support in a chosen representation
space. This space may be the physical input space when its geometry is
meaningful or a learned low-dimensional representation.

Training data for the correction model are constructed through structured
cross-fitting. The dataset is partitioned into folds chosen to create
meaningful separation in the input or representation space. For each fold, an
auxiliary surrogate is trained on the remaining data and evaluated on the
held-out fold, producing approximately out-of-sample errors and query--anchor
configurations that mimic controlled interpolation and extrapolation. These
auxiliary surrogates are used only to generate calibration examples and are
discarded after calibration. GeoQ is then applied to the fixed deployed
surrogate without modifying its architecture or training objective. To assess
whether this learned calibration is applicable at a new query, GeoQ also
computes a feature-space support score that measures how closely the query's
geometric configuration resembles those observed during calibration. This
score identifies regions where the learned correction is supported by the
available calibration data and regions where the resulting error estimate
should be interpreted with greater caution. The resulting estimates are
defined relative to the data-generating process represented by the calibration
data and do not account for model-form discrepancy between different
simulators, experiments, or physical systems
\citep{kennedy2001bayesian,brynjarsdottir2014model}.

The main contributions of this work are the GeoQ framework, a geometry-aware and fully non-intrusive approach to surrogate error estimation, a calibration strategy that learns upper conditional quantiles of error growth from anchor-query relationships generated through structured cross-fitting, and an explicit support score that indicates when an error estimate is being applied outside the region supported by calibration data. We demonstrate the approach on problems ranging from scalar regression and chaotic dynamics to weather forecasting and Richtmyer--Meshkov instability prediction. 
%GeoQ estimates error relative to the same data-generating process used for calibration and does not address full model-form discrepancy between different simulators, experiments, or physical systems~\citep{kennedy2001bayesian,brynjarsdottir2014model}.

The remainder of the paper is organized as follows. Section \ref{sec:related_work} reviews related approaches for uncertainty estimation and surrogate-model error assessment. Section \ref{sec:method} introduces the GeoQ methodology, including the anchor-based error decomposition, geometric feature construction, cross-fitted calibration procedure, conditional quantile model, and feature-space support score. Section \ref{sec:numerical_results} presents the experimental setup and benchmark problems, GeoQ performance, and the comparison with baseline uncertainty-estimation methods. Finally, Section \ref{sec:conclusion} summarizes the main findings, discusses limitations, and outlines directions for future work.

\section{Related work}\label{sec:related_work}

\paragraph{Validation, generalization, and error assessment}
Standard validation workflows assess surrogate reliability through validation loss, test error, cross-validation, and related model-selection procedures~\citep{stone1974cross,hastie2009elements,mohri2018foundations}. These tools are essential for estimating average predictive performance on a held-out distribution and for controlling overfitting during training. Classical statistical learning theory provides a complementary perspective by relating empirical risk, model complexity, and expected generalization error. This includes Vapnik--Chervonenkis (VC) and probably-approximately-correct (PAC) bounds for classification~\citep{vapnik2000nature,valiant1984theory}, Rademacher-complexity and covering-number bounds for real-valued prediction~\citep{bartlett2002rademacher}, algorithmic stability~\citep{bousquet2002stability}, PAC-Bayesian analysis~\citep{mcallester2003pac}, and regularized regression methods such as support vector regression~\citep{smola2004tutorial}. These results are foundational for understanding population-level generalization. However, they do not indicate whether a deployment point is geometrically well supported by the training or calibration data or whether it lies in a regime of extrapolation.

\paragraph{Geometry, applicability domains, and support-aware reliability}
The present work is closest in spirit to methods that connect reliability to geometry, support, or applicability domains. In Gaussian-process and kernel-based models, predictive uncertainty depends on the covariance kernel and on the location of training data~\citep{rasmussen2006gpml}. More broadly, distance-, density-, and novelty-based criteria are often used to determine whether a new query resembles the data used for model construction. These ideas are especially relevant in scientific machine learning, where raw Euclidean distance may fail to capture the structure of the data. High-dimensional states may lie near lower-dimensional manifolds~\citep{tenenbaum2000global,coifman2006diffusion}, exhibit coherent regimes, or vary primarily along active directions~\citep{constantine2015active}. Metric learning provides another perspective by adapting distances to task-relevant notions of similarity~\citep{kulis2013metric}. 

% GeoQ applies these ideas to non-intrusive error estimation by using anchor-relative displacement to model error growth with a conditional quantile estimator, translating geometric information into calibrated error estimates rather than treating it solely as a novelty score.

\begin{table}[t]
    \centering
    \caption{Qualitative comparison of common approaches for uncertainty and error estimation in scientific surrogate models. Applicability-domain awareness indicates whether the method explicitly incorporates distance or support information relative to training or calibration data. OOD-risk detection refers to the ability to flag queries outside the training or calibration support. Non-intrusive indicates whether the method can be applied to a pretrained surrogate without modifying its architecture or training objective. Entries summarize typical implementations; individual variants may differ.\label{tab:MethodComparisonTheoretical}}    
    
    \resizebox{\columnwidth}{!}{%
    \begin{tabular}{|c|c|c|c|c|c|}
        \hline
         \textbf{Method class} & \textbf{Applicability domain/} & \textbf{OOD-risk} & \textbf{Non-intrusive} & \textbf{Training} & \textbf{Inference} \\
         & \textbf{awareness} & \textbf{detection} & & \textbf{Scalability} & \textbf{Scalability} \\
        \hline         
         Bayesian neural network & No & Weak-Moderate & No & Medium & Medium \\
         Deep ensemble & No & Moderate & Yes & Medium$^\ast$ & Medium$^\ast$ \\
         Conformal prediction & No & Weak & Yes & High & High \\
         Generative model & No & Weak & No & Medium & Medium \\
         Gaussian process & Yes & Strong & No & Low & Low \\
         GeoQ & Yes & Strong & Yes & Medium$^\ast$ & High \\
        \hline
    \end{tabular}
    }
    \begin{flushleft}
        \footnotesize{$^\ast$Cost varies with the number of random seeds, ensemble members, or calibration folds used.}
    \end{flushleft}
\end{table}

\paragraph{Uncertainty quantification in scientific surrogates}
There is a growing literature on estimating uncertainties in scientific surrogate models such as neural operators. These methods can be broadly separated according to whether they explicitly encode support or applicability-domain information. \textcolor{black}{Uncertainty-estimation methods without explicit anchor- or support-relative error modeling include surrogates that integrate Bayesian neural networks} \citep{Magnani2025, Lin2023, Garg2023, Zou2024}, ensemble-type approaches \citep{Yang2022, Garg2023a, Mouli2024}, generative models \citep{Bulte2025, Prakash2026}, and conformal prediction \citep{Ma2024, Moya2025}.  Applicability-domain-aware methods include Gaussian-process and kernel-based surrogates \citep{Kumar2024, Kumar2025, Bonneville2024b, Mora2025}. These methods are naturally suited to deployment settings involving distribution shift because they can flag queries that are not well represented by the available data. However, their reliability depends on whether the chosen geometry captures the directions along which surrogate error actually changes. 

\paragraph{Positioning of GeoQ}
GeoQ is designed for \textcolor{black}{applicability-domain-aware, query-specific error estimation} for trained scientific surrogate models. A qualitative comparison \textcolor{black}{of GeoQ with} these other methods for uncertainty quantification in scientific surrogates is presented in Table \ref{tab:MethodComparisonTheoretical}. It preserves the key advantage of geometry-aware methods through an explicit notion of similarity to calibration data, while avoiding the cost of full Gaussian process covariance inference. Additionally, GeoQ is non-intrusive \textcolor{black}{with respect to the deployed surrogate architecture and training objective}, while also allowing for error estimation in pre-trained scientific surrogate models.

% These different approaches for estimating uncertainty in neural surrogates are compared in Table \ref{tab:MethodComparison_Theoretical}. This comparison considers OOD-risk identification as an important attribute that separates applicability domain/distance aware methods with those that do not include such a metric. Furthermore, we also compare the intrusive nature of each approach in terms of application to arbitrary neural surrogate or operators. Intrusive methods are neural surrogate architecture dependent and needs to be redesigned for each surrogate, limiting their applicability in this world of rapidly changing state-of-the-art surrogate method. Non-intrusive methods do not involve any such changes and are more widely applicable. Training and inference scalability identifies the computational expense need to deploy the uncertainty estimation method for more large or larger degrees of freedom problems. As observed for this tabular comparison, our proposed approach preserves the key advantage of GP-based methods, that is an explicit notion of similarity to training data, while avoiding the cost of full GP covariance inference. It also allows the learned error model to capture non-monotonic and direction-dependent relationships between distributional distance and prediction error.

\section{Method}\label{sec:method}

\subsection{Problem formulation}

We consider a dataset $\mathcal{D}=\{(x_i,y_i)\}_{i=1}^n$, where
$x_i\in\mathcal{X}$ denotes the input and $y_i\in\mathscr{Y}$ denotes the
corresponding output. The input space $(\mathcal{X},d_{\mathcal{X}})$ is a
metric space, and the ambient output space $\mathscr{Y}$ may be finite- or
infinite-dimensional. Thus, the inputs and outputs may be vectors, functions,
or other objects equipped with suitable notions of distance.

We assume that the data are generated by an underlying deterministic map
$F:\mathcal{X}\to\mathscr{Y}$ such that
$y_i=F(x_i)$.\footnote{More generally, one may consider stochastic responses
of the form $F(x,\omega)$, where $\omega$ denotes hidden or random variables.
The deterministic setting is assumed throughout this work.}

The range of $F$ may form a structured subset of $\mathscr{Y}$. Let
$\widehat F:\mathcal{X}\to\mathscr{Y}$ denote a surrogate approximation of
$F$. Although $\widehat F$ takes values in the same ambient output space, it
need not preserve the structure induced by $F$ exactly.

%\textcolor{blue}{Let $\mathscr{E}$ be either $\mathbb{R}^m$ or a linear space
%of $\mathbb{R}^m$-valued functions. We equip $\mathscr{E}$ with the partial
%order induced by the positive cone $\mathscr{E}_+$. For vectors
%$a,b\in\mathbb{R}^m$, we %write $a\preceq b$ if %$a_\ell\le b_\ell$ for each
%component $\ell=1,\ldots,m$. For function-valued errors, $a\preceq b$ means
%$a(s)\preceq b(s)$ pointwise, or almost everywhere when the functions are
%defined only up to equivalence. The positive part $[a]_+$ is defined using
%this same componentwise or pointwise order.}

For a prescribed error functional
$\rho:\mathscr{Y}\times\mathscr{Y}\to\mathscr{E}_+$, the prediction error at
$x$ is defined as
\begin{equation}
E(x)=\rho\!\left(F(x),\widehat F(x)\right)\in\mathscr{E}_+,
\qquad x\in\mathcal{X}.
\label{eq:pointwise_error}
\end{equation}
Depending on the choice of $\rho$, $E(x)$ may be a scalar error, a vector of
componentwise errors, or a scalar- or vector-valued error field. For example,
$E(x)=|F(x)-\widehat F(x)|$ defines an absolute error field when the outputs
are functions, while a norm or another summary metric may be used to produce
a scalar error. Throughout the following discussion, inequalities, positive
parts, and other algebraic operations involving $E(x)$ are interpreted
componentwise for vector errors and pointwise for error fields.

% \textcolor{blue}{Throughout the following discussion, inequalities, positive parts, conditional quantiles, and coverage statements involving non-scalar errors are interpreted with respect to this componentwise or pointwise partial order, unless explicitly stated otherwise. Thus, componentwise or pointwise upper quantiles imply marginal error estimates, not simultaneous coverage of an entire vector or field.}

The error $E(x)$ can be evaluated only at inputs for which the reference
response $F(x)$ is available. In particular, the errors available from the
dataset are $E(x_i)=\rho\!\left(y_i,\widehat F(x_i)\right)$ for
$i=1,\ldots,n$. Our objective is to use these observed errors to estimate
$E(x)$ at previously unseen query inputs, including inputs near or outside the
support of the available data.

\textcolor{black}{The method assumes that the calibration data and the test
evaluation are generated by the same reference map or data source. It does
not require the calibration and test queries to be identically distributed;
however, the usefulness of the learned correction depends on whether the test
query geometry is represented in the cross-fitted calibration features and
detected as supported by the safe-region score.}

\subsection{Regularity motivation}

The proposed estimator is motivated by local regularity of the reference and
surrogate maps. We first state the argument in the scalar metric case. This
corresponds to the special case of the error formulation in
Eq.~\eqref{eq:pointwise_error} with
$\mathscr E=\mathbb R$, $\mathscr E_+=\mathbb R_+$, and
$\rho(y,y')=d_{\mathscr Y}(y,y')$.

{\color{black}
\begin{prop}[Anchor-averaged regularity bound]
Let $(\mathcal X,d_{\mathcal X})$ be a metric input space and
$(\mathscr Y,d_{\mathscr Y})$ be a metric output space. Let
$F,\widehat F:\mathcal R\subseteq\mathcal X\to\mathscr Y$ satisfy, for all
$x,x_0\in\mathcal R$,
\[
d_{\mathscr Y}(F(x),F(x_0))\le L_F d_{\mathcal X}(x,x_0),
\qquad
d_{\mathscr Y}(\widehat F(x),\widehat F(x_0))
\le L_{\widehat F}d_{\mathcal X}(x,x_0),
\]
where $L_F,L_{\widehat F}\ge 0$. Define the scalar surrogate error
\[
E(x)=d_{\mathscr Y}(F(x),\widehat F(x)).
\]
Then, for any $x,x_0\in\mathcal R$,
\begin{equation}
E(x)\le E(x_0)+(L_F+L_{\widehat F})d_{\mathcal X}(x,x_0).
\label{eq:regularity_bound}
\end{equation}
Moreover, for any finite anchor set
$\mathcal A(x)\subseteq\mathcal R$,
\begin{equation}
E(x)
\le
\frac{1}{|\mathcal A(x)|}
\sum_{x_a\in\mathcal A(x)} E(x_a)
+
\frac{L_F+L_{\widehat F}}{|\mathcal A(x)|}
\sum_{x_a\in\mathcal A(x)}
d_{\mathcal X}(x,x_a).
\label{eq:anchor_regularity_bound}
\end{equation}
\end{prop}

\begin{proof}
For any $x,x_0\in\mathcal R$, the triangle inequality in
$(\mathscr Y,d_{\mathscr Y})$ gives
\[
\begin{aligned}
E(x)
&=d_{\mathscr Y}(F(x),\widehat F(x))\\
&\le
d_{\mathscr Y}(F(x),F(x_0))
+d_{\mathscr Y}(F(x_0),\widehat F(x_0))
+d_{\mathscr Y}(\widehat F(x_0),\widehat F(x))\\
&\le
E(x_0)+(L_F+L_{\widehat F})d_{\mathcal X}(x,x_0).
\end{aligned}
\]
Applying this inequality with $x_0=x_a$ for each
$x_a\in\mathcal A(x)$ and averaging over the anchor set gives the stated
anchor-averaged bound in Eq.~\eqref{eq:anchor_regularity_bound}.
\end{proof}
}

For non-scalar errors, an analogous componentwise or pointwise bound applies when the corresponding componentwise or pointwise error functional satisfies the triangle inequality and the reference and surrogate maps satisfy the corresponding componentwise or pointwise regularity estimates. \textcolor{black}{In that setting, $L_F$ and $L_{\widehat F}$ should be interpreted as componentwise or pointwise regularity factors, not scalar Lipschitz constants.}

A similar argument applies under H\"older continuity by replacing
$d_{\mathcal X}(x,x_a)$ with $d_{\mathcal X}(x,x_a)^\alpha$ for
$\alpha\in(0,1]$. In practice, the regularity factors are generally unknown,
and bounds based on worst-case values can be overly conservative.

\textcolor{black}{This proposition is used only as a structural motivation for
GeoQ. The learned GeoQ correction is not assumed to be a direct estimate of
the input-space Lipschitz term in Eq.~\eqref{eq:anchor_regularity_bound}.
Instead, GeoQ retains the same structural form as
Eq.~\eqref{eq:anchor_regularity_bound} and learns the correction empirically
from calibration data.}

\subsection{Anchor-based error decomposition}

The regularity bound in Eq.~\eqref{eq:anchor_regularity_bound} motivates an
estimator consisting of an error baseline computed from a suitable set of
anchors $\mathcal{A}(x)$ and a nonnegative correction. GeoQ specifies the
anchor set using a chosen notion of geometric nearness and learns the
correction from calibration data.

The usefulness of nearby reference errors depends on the geometry used to
identify nearby inputs. The simplest choice is distance in the original input
space when $\mathcal{X}$ is equipped with a meaningful norm. For
high-dimensional inputs, however, this distance may not reflect the directions
along which the state or its prediction error varies. We therefore introduce a
representation map $\phi:\mathcal{X}\to\mathscr{Z}$, where $\mathscr{Z}$ is a
normed representation space. The map $\phi$ may be the identity, a learned
feature map, or another representation chosen to encode relevant structure,
such as sensitivity or low-dimensional variation.  \textcolor{black}{The effectiveness of the anchor baseline and the learned correction depends on whether $\phi$ captures geometry that is relevant to
surrogate error.}

For an input set $\mathcal{I}\subseteq\mathcal{X}$, a representation map
$f:\mathcal{X}\to\mathscr{Z}_f$, and a positive integer $k$, define
\[
\mathcal{N}_{f,k}^{\mathcal{I}}(x)
:=
\operatorname{kNN}^{(k)}_{x_j\in\mathcal{I}}
\left\|f(x)-f(x_j)\right\|_{\mathscr{Z}_f}.
\]
\textcolor{black}{as set of $k$ nearest
points to $x$ in the representation induced by $f$. The ties in the definition of $\mathcal{N}_{f,k}^{\mathcal{I}}(x)$ are broken deterministically.}

Let
$\mathcal{X}_{\mathcal{D}}=\{x_i:(x_i,y_i)\in\mathcal{D}\}$
denote the available inputs, and let $k_a$ be the number of anchors. The
anchor-averaged error at a query $x$ is
\begin{equation}
\overline E_{k_a}(x)
=
\frac{1}{k_a}
\sum_{x_j\in
\mathcal{N}_{\phi,k_a}^{\mathcal{X}_{\mathcal{D}}}(x)}
E(x_j).
\label{eq:anchor_average}
\end{equation}
The parameter $k_a$ controls the neighborhood over which the observed anchor
errors are averaged.

Using this anchor baseline, GeoQ estimates the query error as
\begin{equation}
\widehat E(x)
=
\overline E_{k_a}(x)
+
\widehat R_{k_a}(x),
\label{eq:anchor_error_decomposition}
\end{equation}
where $\widehat R_{k_a}(x)\in\mathscr{E}_+$ is a learned nonnegative
correction. The corresponding calibration target is
\begin{equation}
R_{k_a}(x)
=
\left[E(x)-\overline E_{k_a}(x)\right]_+.
\label{eq:anchor_increment}
\end{equation}
The target $R_{k_a}(x)$ can be evaluated only when the reference response at
$x$ is available.

In Eq.~\eqref{eq:anchor_regularity_bound}, the correction depends on the
query--anchor distances and an unknown worst-case regularity factor. GeoQ
replaces this prescribed correction with a data-driven model whose inputs
describe the geometry of the query relative to the available data. We % approximate 
write the nonnegative correction as
\[
\widehat R_{k_a}(x)
\approx
g_\eta\!\left(\psi_{k_a,k}(x)\right),
\qquad
g_\eta:\mathscr{P}\to\mathscr{E}_+,
\]
where $\psi_{k_a,k}(x)\in\mathscr{P}$ is a geometric feature vector that may
include representation-space displacements, distances, and local-support
information. The feature map $\psi_{k_a,k}$ and its feature space
$\mathscr{P}$ are defined in the following subsection.

Figure~\ref{cartoon_anchor} illustrates this decomposition. The
anchor-averaged error provides a local baseline from observed errors, and the
learned correction estimates the remaining geometry-dependent error
increment.

% A deterministic bound that holds for all queries would generally require
% worst-case regularity information and may be overly conservative. GeoQ instead
% seeks a probabilistic upper estimate by modeling an upper conditional quantile
% of the correction term. For a prescribed quantile level $\tau\in(0,1)$, the
% correction model is trained so that
% \begin{equation}
% g_\eta(\psi)
% \approx
% Q_\tau\!\left(
% R_{k_a}(x)
% \,\middle|\,
% \psi(x,\mathcal{A}_{k_a}(x))=\psi
% \right).
% \label{eq:anchor_quantile_target}
% \end{equation}
% The resulting estimator is intended to satisfy
% $\mathbb{P}(E(x)\leq\widehat E(x))\approx\tau$ within regions supported by
% the calibration data, as measured by the support score introduced in
% Section~\ref{subsec:safe-region}. For non-scalar errors, this is a marginal
% componentwise or pointwise coverage statement and does not imply simultaneous
% coverage of the entire vector or field.

\begin{figure}[H]
\centering

\begin{minipage}[c]{0.54\textwidth}
    \centering
    \subfloat[Anchor neighborhood in representation space]{
        \includegraphics[width=\linewidth]{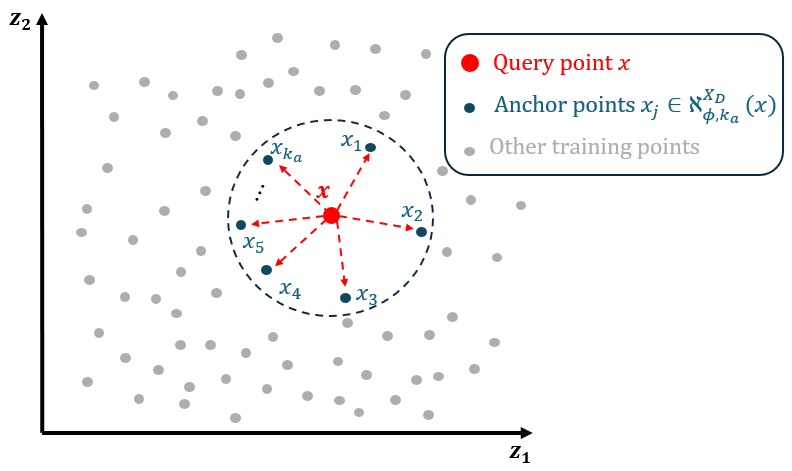}
    }
\end{minipage}
\hfill
\begin{minipage}[c]{0.42\textwidth}
    \centering
    \subfloat[Anchor-averaged error baseline]{
        \includegraphics[width=\linewidth]{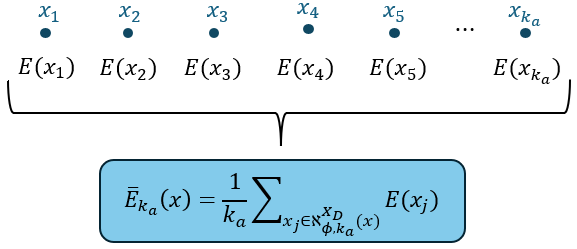}
    }

    \vspace{1.0em}

    \subfloat[Error decomposition at the query point]{
        \includegraphics[width=\linewidth]{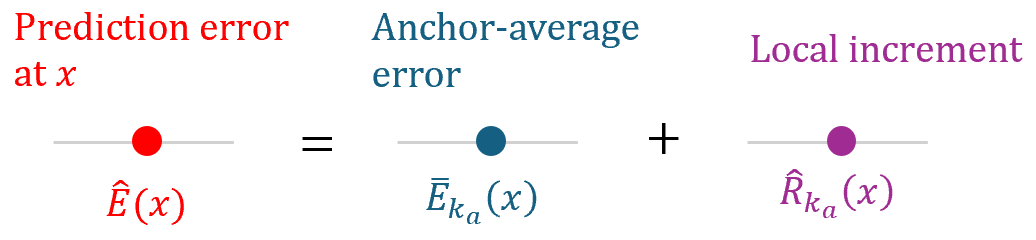}
    }
\end{minipage}

\caption{Schematic of the anchor-based GeoQ error model. A query point is compared with its local anchors in the representation space induced by $\phi$. The anchor errors define a local baseline $\overline{E}_{k_a}(x)$, and GeoQ learns a nonnegative correction $\widehat{R}_{k_a}(x)$ to estimate the query error.}
\label{cartoon_anchor}
\end{figure}

\subsection{Geometric feature construction}

To learn a correction that can transfer beyond the specific calibration points, GeoQ represents each query through its geometry relative to nearby anchors. We choose to construct features from anchor-relative displacement and local support distance in a chosen embedding space.
The anchor-relative displacement feature is defined as the average displacement from the local anchor set to the query
\begin{align}
\Delta z_{k_a}(x)
=
\frac{1}{k_a}
\sum_{x_j\in\mathcal{N}_{\phi,k_a}^{\mathcal{X}_\mathcal{D}}(x)}
\left(\phi(x)-\phi(x_j)\right)
\in\mathscr{Z}.
\label{eq:dz}
\end{align}
To measure local input-space support, we define a normalization map
$\nu:\mathcal{X}\to\mathscr{Z}_{\nu}$, where $\mathscr{Z}_{\nu}$ is a normed space.
The map $\nu$ applies a chosen normalization based on the training data, such as
coordinatewise standardization for finite-dimensional inputs. Furthermore,
$\mathcal{N}_{\nu,k}^{\mathcal{X}_\mathcal{D}}(x)$ denotes the $k$ nearest
neighbors to $x$ in the representation induced by $\nu$. The normalized
input-space support feature is then defined as
\begin{align}
\Delta d_k(x)
=
\frac{1}{k}
\sum_{x_j\in\mathcal{N}_{\nu,k}^{\mathcal{X}_\mathcal{D}}(x)}
\left\|\nu(x)-\nu(x_j)\right\|_{\mathscr{Z}_{\nu}}.
\label{eq:dd}
\end{align}
This quantity measures the average distance from the query to its $k$-nearest neighbors in normalized input space.

For each query point, the geometric feature vector of GeoQ is defined as
\begin{align}
\psi_{k_a,k}(x)
=
\left(\Delta z_{k_a}(x),\Delta d_k(x)\right)
\in\mathscr{P},
\qquad
\mathscr{P}:=\mathscr{Z}\times\mathbb{R}.
\label{eq:psi}
\end{align}
where $k_a$ controls the number of anchors used to construct the anchor-averaged error baseline, while $k$ controls the scale of the local support-density descriptor. This formulation expresses both components of $\psi_{k_a,k}$ as nearest-neighbor quantities induced by specified embeddings, making the feature construction systematic and explicitly tied to the geometry of the calibration data. \textcolor{black}{For finite-dimensional implementations, $\psi_{k_a,k}(x)$
is vectorized and equipped with the Euclidean norm after the prescribed
feature scaling.} 
% {In the numerical examples, this corresponds to the standard Euclidean norm on the concatenated feature vector $(\Delta z_{k_a}(x),\Delta d_k(x))$.}

Since the nearest-neighbor sets depend on the query location, the resulting displacement features are generally not globally smooth. The parameters $k_a$ and $k$ control the geometric scale at which the data are summarized: smaller values emphasize localized anchor-relative structure, whereas larger values yield smoother, increasingly global summaries of the calibration data.

Additional features may be incorporated to capture local sensitivity of the surrogate, such as gradient-based descriptors derived from $\nabla\widehat F(x)$ or related quantities. While such features can, in principle, improve the characterization of error growth, their practical utility may be limited by the reliability of surrogate gradients, particularly in extrapolative regions, as well as by increased feature sparsity in high-dimensional settings. Accordingly, we restrict attention to the two geometry-based features above, although the framework readily accommodates richer feature representations.

\subsection{Conditional quantile increment model}

The regularity argument motivates an upper correction to the anchor-averaged
error, but a deterministic bound would require worst-case regularity factors
that are generally unavailable. We therefore estimate an upper conditional
quantile of the anchor-relative error increment.

We seek a calibrated upper estimate of $R_{k_a}(x)$ by learning a nonnegative
function $g_\eta:\mathscr{P}\to\mathscr{E}_+$ that predicts a high quantile of
$R_{k_a}(x)$ from $\psi_{k_a,k}(x)$. For a target quantile level
$\tau\in(0,1)$, the function $g_\eta(\psi)$ is trained to approximate the
corresponding conditional quantile of the error increment:
\begin{align}
g_\eta(\psi)
\approx
Q_\tau\left(
R_{k_a}(x)
\mid
\psi_{k_a,k}(x)=\psi
\right).
\end{align}
Here, $Q_\tau(\cdot)$ denotes the conditional $\tau$-quantile operator. For
non-scalar errors, the conditional quantile is evaluated componentwise or
pointwise. Thus, within calibration-supported regions of feature space,
$g_\eta(\psi)$ is intended to satisfy
$\mathbb{P}(R_{k_a}(x)\leq g_\eta(\psi)\mid
\psi_{k_a,k}(x)=\psi)\approx\tau$.  \textcolor{black}{The quality of this approximation depends on the
representativeness of the cross-fitted calibration tuples, the expressiveness
of $g_\eta$, and the extent to which $\psi_{k_a,k}$ captures the
error-relevant geometry.} 
% {Thus, the learned conditional quantile should be interpreted as an empirical calibration model rather than as a distribution-free coverage guarantee.}

To preserve the anchor-based construction, we parameterize the increment
model as
\begin{align}
g_\eta(\psi)
=
\|\psi\|_{\mathscr{P}}\,h_\eta(T(\psi)),
\end{align}
where $T$ denotes a fixed feature-standardization map that is fit on the
calibration feature set and kept fixed during training and inference, and
$h_\eta(T(\psi))\in\mathscr{E}_+$ is a nonnegative learned model. This
construction guarantees $g_\eta(\psi)\in\mathscr{E}_+$ and $g_\eta(0)=0$, so
the predicted increment vanishes when the geometric feature vector vanishes.
The factor $\|\psi\|_{\mathscr{P}}$ introduces an explicit dependence on the
magnitude of the geometric separation represented by $\psi$.

\subsection{Cross-fitted calibration data}

Since the quantities $E(x)$ and $R_{k_a}(x)$ are not observable for unseen testing points, we estimate their behavior using a cross-fitting procedure. The training dataset is partitioned into $K_f$ folds
$\{\mathcal{D}^{(\ell)}\}_{\ell=1}^{K_f}$ where $\mathcal{D}=\bigcup_{\ell=1}^{K_f}\mathcal{D}^{(\ell)}$. For each fold $\ell$, we train an auxiliary surrogate $\widehat{F}^{(-\ell)}$ on $\mathcal{D}\setminus\mathcal{D}^{(\ell)}$ and evaluate it on the held-out fold $\mathcal{D}^{(\ell)}$ to obtain approximately out-of-sample errors. The auxiliary surrogates use the same architecture and training procedure as the final surrogate. The representation map $\phi$ used to construct GeoQ features is fixed and shared across all folds, for example by using the identity map, a pretrained encoder, or the encoder of the final surrogate held fixed during calibration. This ensures that signed displacement features are computed in a common coordinate system.

We denote $\mathcal{X}^{(-\ell)}_{\mathcal{D}}=\{x_i:(x_i,y_i)\in\mathcal{D}\setminus\mathcal{D}^{(\ell)}\}$ the set of fold-training inputs used for anchor selection when the $\ell$-th fold is held out. For each held-out pair $(x_i,y_i)\in\mathcal{D}^{(\ell)}$, the cross-fitted error $e_i^{\mathrm{cf}}$ is obtained from Eq.~\eqref{eq:pointwise_error} by replacing $\widehat F$ with $\widehat F^{(-\ell)}$ and using $y_i$ as the reference response. The corresponding anchor-averaged error $\overline e_{k_a,i}^{\mathrm{cf}}$ is obtained from Eq.~\eqref{eq:anchor_average} by restricting the anchor search to
$\mathcal{N}_{\phi,k_a}^{\mathcal{X}_{\mathcal{D}}^{(-\ell)}}(x_i)$
and replacing $E(x_j)$ with
$\rho\!\left(y_j,\widehat F^{(-\ell)}(x_j)\right)$. The calibration increment is then
$r_i^{\mathrm{cf}}=\left[e_i^{\mathrm{cf}}-\overline e_{k_a,i}^{\mathrm{cf}}\right]_+$. \textcolor{black}{This construction makes the query residual
$e_i^{\mathrm{cf}}$ out-of-sample with respect to the auxiliary surrogate
$\widehat F^{(-\ell)}$, since $x_i$ belongs to the held-out fold
$\mathcal D^{(\ell)}$. In contrast, the anchor residuals entering
$\overline e_{k_a,i}^{\mathrm{cf}}$ are evaluated on fold-training points
$x_j\in\mathcal X_{\mathcal D}^{(-\ell)}$ and are therefore in-sample for
$\widehat F^{(-\ell)}$. This is intentional: at deployment, the query error is
unknown, whereas the anchor errors are observed on the available anchor or
calibration set.}

The cross-fitted geometric features are obtained from
Eqs.~\eqref{eq:dz}--\eqref{eq:psi} by replacing
$\mathcal{X}_{\mathcal{D}}$ with the fold-training input set
$\mathcal{X}_{\mathcal{D}}^{(-\ell)}$ in the nearest-neighbor searches. The
representation map $\phi$ and normalization map $\nu$ are fixed and shared
across all folds. 
Thus, $\Delta z_{k_a}^{\mathrm{cf}}(x_i)$ and
$\Delta d_k^{\mathrm{cf}}(x_i)$ select neighbors only from
$\mathcal{X}_{\mathcal{D}}^{(-\ell)}$, while the normalized distances are
computed in a common coordinate system across folds. Together these define
$\psi_i=\psi_{k_a,k}^{(-\ell)}(x_i)
=\left(\Delta z_{k_a}^{\mathrm{cf}}(x_i),
\Delta d_k^{\mathrm{cf}}(x_i)\right)\in\mathscr{P}$, which yields the
calibration dataset
$\mathcal{C}_{\mathrm{GeoQ}}
=\{(\psi_i,r_i^{\mathrm{cf}})\}_{i=1}^n$.

Cross-fitting mitigates training-set optimism and provides approximately out-of-sample quantities for learning error growth. The folds may be constructed via random partitioning, clustering in principal component or latent feature space, or problem-specific strategies such as temporal or regime-based splits. The appropriate fold construction depends on the intended deployment regime. For interpolative settings, random folds are often suitable because they produce calibration pairs representative of near-support predictions with relatively small query-anchor displacements. For extrapolative or distribution-shifted settings, random folds may be insufficient because they primarily generate small-displacement calibration pairs and may provide limited variation in the increment target $r_i^{\mathrm{cf}}$. In such cases, cluster-based, temporal-based, or regime-based folds can expose the calibration procedure to a broader range of query-anchor separations, including small, intermediate, and large displacements. The choice of $K_f$ also induces a tradeoff: smaller $K_f$ can increase geometric diversity in the held-out samples but may reduce fold-surrogate accuracy, whereas larger $K_f$ improves fold-surrogate fidelity at the cost of reduced coverage in displacement space.

After constructing $\mathcal{C}_{\mathrm{GeoQ}}$, for $\mathscr{E}=\mathbb{R}^m$, the increment model is trained using the following componentwise pinball loss:
\begin{align}
\mathcal{L}(\eta) =\dfrac{1}{nm} \sum_{i=1}^n\sum_{j=1}^m \max\left(\tau\left(r_{i,j}^{\mathrm{cf}}-g_{\eta,j}(\psi_i)\right), (\tau-1)\left(r_{i,j}^{\mathrm{cf}} - g_{\eta,j}(\psi_i)\right) \right).
\end{align}
For function-valued errors, the same loss is applied pointwise over the discretized output domain and averaged. The pinball loss penalizes underestimation by factor $\tau$ and overestimation by factor $1-\tau$. When $\tau$ is high, the loss encourages the learned correction to approximate an upper conditional quantile of the anchor-relative error increment. The final error estimator is then defined as
\begin{align}
\widehat E(x)=\overline E_{k_a}(x)+g_\eta(\psi_{k_a,k}(x)).
\end{align}

\begin{figure}[t]
  \centering
  \subfloat[Fold construction in feature space]{
  \adjustbox{valign=c}{%
    \includegraphics[width=0.33\textwidth]{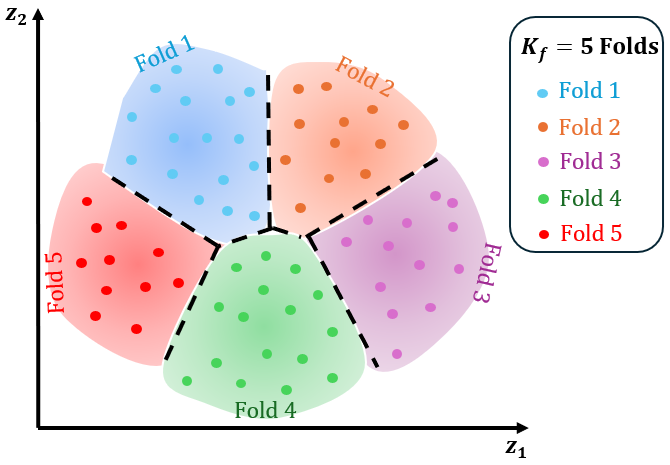}
  }
}
\subfloat[Held-out queries and training local anchors]{
  \adjustbox{valign=c}{%
    \includegraphics[width=0.33\textwidth]{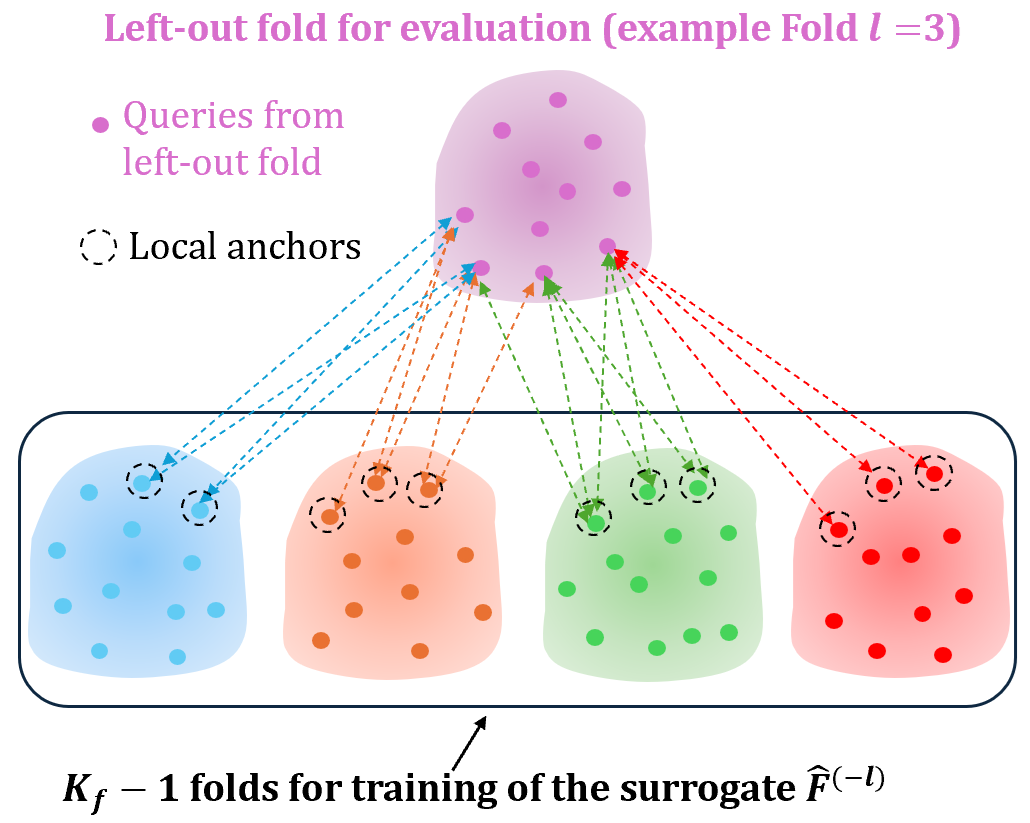}
  }
}
\subfloat[Calibration set construction]{
  \adjustbox{valign=c}{%
    \includegraphics[width=0.33\textwidth]{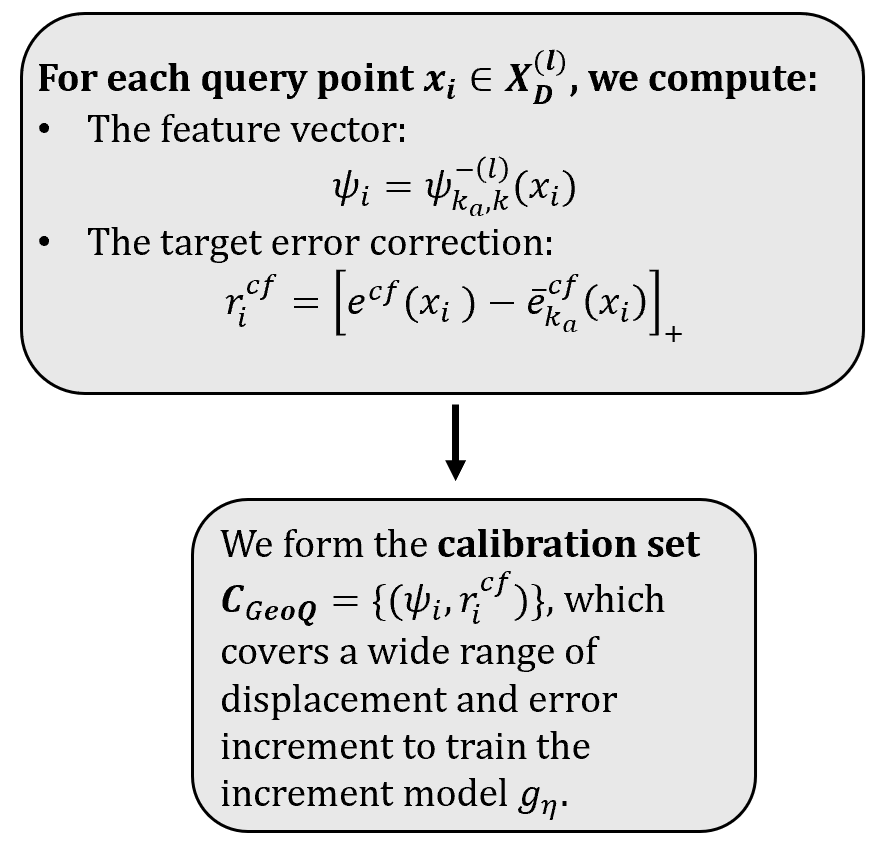}
  }
}
  \caption{Schematic of the GeoQ cross-fitted calibration procedure.}\label{cartooncrossfit}
\end{figure}

\subsection{Safe-region score}\label{subsec:safe-region}

To assess whether the estimated error is being applied in a
calibration-supported region, we define a density-based support score directly
in the calibration feature space. Let
$\Psi_{\rm cal}=\{\psi_i\}_{i=1}^n\subset\mathscr P$
denote the collection of cross-fitted calibration features.

\textcolor{black}{Let $T:\mathscr P\to\mathbb R^q$ denote the fixed
feature-standardization map fit on the calibration feature set
$\Psi_{\rm cal}$. This is the same standardization map used by the increment
model. We define the standardized feature-space distance}
{\color{black}
\[
d_T(\psi,\psi')
=
\left\|T(\psi)-T(\psi')\right\|_2 .
\]
}
% under the distance $d_T$.
\textcolor{black}{Let $K_{\mathrm{safe}}$ be the number of calibration-feature
neighbors used to compute the support score. For a query feature vector
$\psi\in\mathscr P$, let
$\mathcal N^{\Psi_{\rm cal}}_{T,K_{\mathrm{safe}}}(\psi)$ denote the set of
$K_{\mathrm{safe}}$ nearest calibration features to $\psi$ in the representation induced by $T$. The support score is then defined by}
\begin{align}
S(\psi)
&=
\frac{1}{K_{\mathrm{safe}}}
\sum_{\psi_j\in
\mathcal N^{\Psi_{\rm cal}}_{T,K_{\mathrm{safe}}}(\psi)}
d_T(\psi,\psi_j)
\label{eq:safe_score_standardized}
\end{align}

Using the standardized distance $d_T$ makes the safe-region criterion less sensitive
to the relative numerical scales of the displacement and support-distance
features and aligns the support score with the feature representation used by
the increment model $g_\eta$.
Small values of $S(\psi)$ indicate that the query geometry is well represented
by the calibration data, whereas large values correspond to sparse or
previously unseen feature configurations.

% We employ a feature-space kNN score rather than global covariance-based
% metrics such as Mahalanobis distance because the calibration feature
% distribution may be highly nonlinear, multimodal, or anisotropic.
% Covariance-based approaches implicitly assume an approximately elliptical
% feature distribution and may fail to characterize regions of valid support in
% such settings. In contrast, the kNN-based criterion adapts locally to the
% geometry and density of the calibration feature cloud without imposing global
% distributional assumptions.

\textcolor{black}{When support scores are computed for calibration features in
order to set the safe-region threshold, we use a leave-one-out version of the
score. For each calibration feature $\psi_i\in\Psi_{\rm cal}$}
\begin{align}
S_{\rm cal}(\psi_i)
=
\frac{1}{K_{\mathrm{safe}}}
\sum_{\psi_j\in
\mathcal N^{\Psi_{\rm cal}\setminus\{\psi_i\}}_{T,K_{\mathrm{safe}}}(\psi_i)}
d_T(\psi_i,\psi_j).
\label{eq:safe_score_calibration_loo}
\end{align}
\textcolor{black}{The exclusion of $\psi_i$ from its own neighbor set prevents
the calibration score from being artificially reduced by zero self-distance.
The safe threshold is then selected as an empirical high quantile of these
leave-one-out calibration scores, for example}
{\color{black}
\begin{align}
\tau_{\mathrm{safe}}
=
\widehat Q_{\alpha_{\mathrm{safe}}}
\left(
\{S_{\rm cal}(\psi_i)\}_{i=1}^n
\right),
\qquad
\alpha_{\mathrm{safe}}=0.95 ,
\label{eq:safe_threshold_loo}
\end{align}
}
\noindent where $\widehat Q_{\alpha_{\mathrm{safe}}}$ denotes the
empirical $\alpha_{\mathrm{safe}}$-quantile. At test time, $S(\psi)$ is
computed using the full calibration feature cloud $\Psi_{\rm cal}$, since the
test feature is not a member of the calibration set. The near-support, or safe, region is defined as
\begin{align}
\Omega_{\mathrm{safe}}
=
\{\psi\in\mathscr P:S(\psi)\le \tau_{\mathrm{safe}}\}.
\end{align}
This validity criterion distinguishes regions where the learned conditional
quantile model remains supported by calibration data from regions where the
geometric configuration becomes insufficiently represented. Coverage within the safe region still depends on the quality of the learned conditional quantile model and on the representativeness
of the cross-fitted calibration features.

\subsection{Algorithm summary and implementation choices}

The full GeoQ workflow, including calibration, error-estimator training, and test-time inference, is summarized in Algorithm \ref{alg:geoq}. \textcolor{black}{When a pretrained deployed surrogate is already available, GeoQ treats that surrogate as fixed and calibrates an error estimator around it. When no such surrogate is supplied, the same workflow can first train the deployed surrogate and then construct the GeoQ calibration model.} The main design choices are the representation map $\phi$, the anchor-neighborhood size $k_a$, the support-neighbor size $k$, and the fold construction. The representation map $\phi$ should reflect the notion of similarity most relevant to surrogate error. Learned encoder features can be useful when the data lie near a latent manifold, whereas the identity map $\phi(x)=x$ is appropriate when physical input-space proximity already provides a meaningful geometry. The anchor size $k_a$ determines the bias-variance tradeoff of the local error baseline. Smaller values preserve locality but can be sensitive to individual anchors, while larger values stabilize the baseline at the cost of possible oversmoothing. The support-neighbor size $k$ sets the scale at which local density is measured. The fold construction should also be chosen according to the intended deployment regime. Random folds are often appropriate for interpolative testing because they produce near-support calibration errors. Cluster-based or temporal-based folds are more suitable for extrapolative testing because they create pseudo-OOD calibration tuples with larger query-anchor displacements. The pinball loss is used because the goal is to estimate an upper error quantile rather than the conditional mean error. Unless otherwise stated, these quantities are fixed before test evaluation, and their sensitivity is examined for the KS equation in \ref{append:ablation}.

\begin{algorithm}[H]
\caption{GeoQ calibration and inference}
\label{alg:geoq}
\begin{algorithmic}[1]
\Require \textcolor{black}{Training data $\mathcal{D}=\{(x_i,y_i)\}_{i=1}^n$, either a trained deployed surrogate $\widehat F$ or a surrogate architecture to be trained, optional representation map $\phi$, quantile level $\tau$, anchor size $k_a$, support-neighbor size $k$, safe-score neighbor size $K_{\mathrm{safe}}$, and number of folds $K_f$.}

\vspace{.15cm}

\Statex \textbf{Specify deployed surrogate}
\State \textcolor{black}{If a trained deployed surrogate $\widehat F$ is supplied, keep it fixed. Otherwise, train the deployed surrogate $\widehat F$ on $\mathcal D$ using the specified surrogate architecture.}

\vspace{.15cm}

\Statex \textbf{Geometry and observed anchor errors}
\State Define $\mathcal X_{\mathcal D}=\{x_i:(x_i,y_i)\in\mathcal D\}$.
\State Use the supplied $\phi$; otherwise define $\phi$ from a fixed representation of $\widehat F$ or set $\phi=\operatorname{Id}_{\mathcal X}$.
\State Fit $\nu$ using $\mathcal X_{\mathcal D}$ and compute $E(x_i)=\rho(y_i,\widehat F(x_i))$ for all $(x_i,y_i)\in\mathcal D$.

\vspace{.15cm}

\Statex \textbf{Cross-fitted calibration data}
\State Partition $\mathcal D$ into folds $\{\mathcal D^{(\ell)}\}_{\ell=1}^{K_f}$.
\For{$\ell=1,\ldots,K_f$}
    \State Train $\widehat F^{(-\ell)}$ on $\mathcal D\setminus\mathcal D^{(\ell)}$ and define $\mathcal X_{\mathcal D}^{(-\ell)}=\{x_i:(x_i,y_i)\in\mathcal D\setminus\mathcal D^{(\ell)}\}$.
    \State For each $(x_i,y_i)\in\mathcal D^{(\ell)}$, compute $e_i^{\mathrm{cf}}$, $\overline e_{k_a,i}^{\mathrm{cf}}$, $\psi_i$, and $r_i^{\mathrm{cf}}=[e_i^{\mathrm{cf}}-\overline e_{k_a,i}^{\mathrm{cf}}]_+$ using $\mathcal X_{\mathcal D}^{(-\ell)}$.
\EndFor
\State Form $\mathcal C_{\mathrm{GeoQ}}=\{(\psi_i,r_i^{\mathrm{cf}})\}_{i=1}^n$ and $\Psi_{\mathrm{cal}}=\{\psi_i\}_{i=1}^n$.

\vspace{.15cm}

\Statex \textbf{Correction-model calibration}
\State Fit the feature-standardization map $T$ on $\Psi_{\mathrm{cal}}$ and keep $T$ fixed during training and inference.
\State Train $g_\eta$ on $\mathcal C_{\mathrm{GeoQ}}$ using the componentwise pinball loss at level $\tau$.

\vspace{.15cm}

\Statex \textbf{Inference}
\For{each test query $x$}
    \State Compute $\overline E_{k_a}(x)$ and $\psi_{k_a,k}(x)$ using $\mathcal X_{\mathcal D}$.
    \State Compute the standardized support score using Eq.~\eqref{eq:safe_score_standardized}.
    \State Return $\widehat E(x)=\overline E_{k_a}(x)+g_\eta(\psi_{k_a,k}(x))$ and $\operatorname{safe}(x)=\mathbf 1\{S(\psi_{k_a,k}(x))\leq\tau_{\mathrm{safe}}\}$.
\EndFor
\end{algorithmic}
\end{algorithm}

% The proposed approach estimates surrogate error by combining anchor-averaged training errors with a learned geometry-aware conditional quantile correction, while controlling validity through support-based scores. We refer to this framework as \textit{GeoQ} (\textit{Geo}metry-Aware Conditional \textit{Q}uantile Error Estimation), reflecting its use of representation-space geometry together with conditional quantile modeling for surrogate error estimation.

\section{Numerical results}\label{sec:numerical_results}
In this section, we assess the performance of the proposed error estimation framework (GeoQ) on four representative problems: the Forrester function, the Kuramoto–Sivashinsky equation, the WeatherBench dataset, and the Richtmyer–Meshkov instability. Detailed descriptions of the surrogate models (and their accuracy) and increment models, including network architectures, optimization procedures, and training configurations, are provided in \ref{append:model_architecture}.

In all considered problems, the pointwise prediction error is defined as the absolute difference between the reference solution and the surrogate prediction
\begin{align}
    E(x)=|F(x)-\hat{F}(x)|.
\end{align}

For all experiments, the calibration dataset is constructed through a leave-one-cluster-out procedure unless otherwise noted. Specifically, the full training set is partitioned into $K_f$ subsets using a $K$-means decomposition of the input space. For each calibration experiment, one cluster is excluded from surrogate training and treated as a held-out fold, while the remaining clusters are used for model training. This procedure produces controlled pseudo-out-of-distribution (pseudo-OOD) samples that emulate extrapolative predictions relative to the effective training manifold of the surrogate model.

Unless otherwise specified, we use the same default GeoQ configuration across all test cases. The anchor neighborhood size is fixed to $k_a=1$, corresponding to the single-nearest-anchor construction, and the local input-space support feature $\Delta d_k$ is computed using $k=20$ nearest neighbors. The GeoQ feature vector is defined as $\psi_{k_a,k}=\left[\Delta z_{k_a},\ \Delta d_k\right]$ where $\Delta z_{k_a}$ is the representation-space displacement between the query and its local anchor. The number of cross-fitting folds $K_f$ is problem-dependent and is specified separately for each benchmark. We provide ablation studies on the sensitivity of these design choices for the KS equation in \ref{append:ablation}.

To quantitatively evaluate the quality of the uncertainty estimates in Section~\ref{sec_quantitative}, we report several complementary metrics. Let $y_i$ denote the reference output, $\widehat{y}_i=\widehat{F}(x_i)$ denote the corresponding surrogate prediction, $E_i$ denote the true surrogate error, and $\widehat{E}_i$ the predicted error estimate. For field-valued outputs, the index $i$ may represent either a sample index or a flattened sample-coordinate index over all evaluated output locations. We consider the following metrics:
\begin{itemize}
    \item \textit{Pinball loss}: 
$\mathcal{L}_{\tau}
=\frac{1}{N}\sum_{i=1}^N\max\left(\tau(E_i-\widehat E_i),(\tau-1)(E_i-\widehat E_i)\right)$ where $\tau\in (0,1)$ is the target quantile level.
    \item \textit{Coverage}: $\frac{1}{N}\sum_{i=1}^N\mathbf{1}\left\{\left|y_i-\widehat y_i\right|\leq\widehat E_i\right\}$
    % $\frac{1}{N}\sum_{i=1}^N\mathbf 1\{E_i \le \widehat E_i\}$
\item \textit{Correlation}: $\mathrm{Corr}(E,\widehat E)
=\frac{\mathrm{Cov}(E,\widehat E)}{\sigma_E \sigma_{\widehat E}}$
\end{itemize}

The \textit{pinball loss} evaluates the quality of the estimated upper quantiles. A lower pinball loss indicates more accurate conditional quantile estimation. In this work, we choose $\tau=0.95$. The \textit{coverage} measures the fraction of samples for which the predicted error bound exceeds the true error. Higher coverage indicates more conservative uncertainty estimates. The \textit{correlation} metric evaluates the Pearson correlation coefficient between the predicted and true errors. This metric measures how well the estimated uncertainty tracks the relative spatial or temporal variations of the true error. We also report training time and inference time. Inference time denotes the total time required to produce both the surrogate prediction and the associated error estimate over the test set. Unless otherwise stated, all training and inference experiments are performed on a single NVIDIA GeForce RTX 2080 Ti GPU.

% \textcolor{blue}{Avi, could you add a sentence specifying the computational resources for your testing case?}

% For spatially distributed fields, we additionally report the \textit{structural similarity index} (SSIM), which measures structural agreement between the predicted and true error fields by comparing local luminance, contrast, and spatial structure. Higher SSIM values indicate improved preservation of coherent spatial error patterns.

\subsection{Forrester function}
We first evaluate the proposed GeoQ method on the one-dimensional Forrester function \citep{forrester2008engineering}
\begin{align}
u(x)= (6x-1)^2\sin(12x-4), \quad \text{for } x\in [0,1],
\end{align}
which is a standard benchmark for nonlinear regression. Despite its low dimensionality, the function exhibits strong nonlinearity and rapidly varying curvature, making it suitable for testing error estimation methods. The surrogate model $\hat{F}$ is a neural network trained on data restricted to the interval $\mathcal{X}_{train}=[0.25,0.75]$ using $1200$ uniformly spaced training points. Evaluation is performed on $4000$ uniformly spaced test points over the full domain $\mathcal{X}_{test}=[0,1]$. Thus, $[0.25,0.75]$ is treated as \textcolor{black}{the} in-distribution region, while $[0,0.25)\cup(0.75,1]$ is used to assess extrapolative behavior. 

The surrogate model $\widehat{F}$ is a fully connected neural network consisting of an encoder and a prediction head. The encoder maps the input $x\in\mathbb{R}$ to a latent feature vector $\phi(x)\in\mathbb{R}^p$, where $p=16$ in the present experiment. The prediction head then maps $\phi(x)$ to the surrogate prediction $\widehat F(x)$.

The latent representation $\phi(x)$ from the final trained surrogate is held fixed and used to define the geometry underlying the anchor-based error estimator. In particular, anchor neighborhoods and displacement features are computed in this shared latent space. The increment model $g_\eta:\mathbb{R}^{p+1}\to\mathbb{R}_+$ is a fully connected neural network that takes the latent displacement feature together with the support-distance feature as input. The model is trained using the pinball loss at quantile level $\tau=0.95$.

To generate calibration data, we construct a sequence of $K=10$ nested training subsets inside the full training interval $[0.25,0.75]$. The nested subsets progressively shrink toward the center of the training interval. For each subset, a stage surrogate is trained on the restricted support and evaluated on points outside that support but still within the full interval. This procedure creates controlled pseudo-OOD examples that mimic extrapolation relative to the stage-training domain.

\begin{figure}[H]
  \centering
{\includegraphics[width=0.7\textwidth]{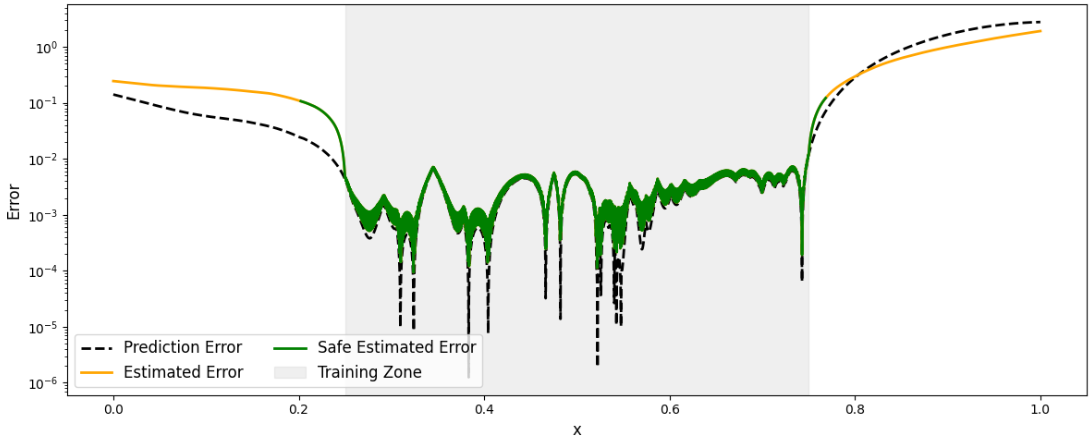}}
  \caption{\textit{Forrester equation:} Visualization of the surrogate prediction error and the estimated error by the proposed GeoQ approach.}\label{forrester_visu}
\end{figure}

Figure \ref{forrester_visu} compares the true surrogate prediction error with the error estimates produced by GeoQ. It can be seen that inside the training region, the surrogate error remains relatively small, and the proposed estimator captures the local error variations while maintaining stable behavior across oscillatory regions. Near the boundaries of the training support, the true error increases rapidly as the surrogate enters extrapolative regimes. The estimated error follows this growth trend, indicating that the learned geometric features successfully capture the increasing uncertainty away from the training support. The safe-region estimate corresponds to predictions restricted to samples identified as reliable by the feature-space kNN criterion. Within this region, the predicted error remains close to the true error envelope while avoiding excessive conservativeness. These results demonstrate that the proposed GeoQ model provides accurate local calibration within the learned safe-region and produces reasonable uncertainty growth under extrapolation.

% The representation map $\phi$ is realized using two hidden layers with $64$ ReLU units each, followed by a linear latent layer. The surrogate prediction is then obtained through a regression head composed of one hidden layer with $32$ ReLU units and a final linear output layer:
% $$
% x \xrightarrow{\phi} z\in\mathbb{R}^{p}
% \xrightarrow{h} \hat{u}(x)
% $$
% The surrogate is trained using mean-squared error loss with the Adam optimizer, learning rate $10^{-3}$, batch size $32$, and $250$ training epochs.

\subsection{Kuramoto–Sivashinsky equation}
We assess the proposed GeoQ method on the one-dimensional Kuramoto–Sivashinsky (KS) equation \citep{sivashinsky1977nonlinear,kuramoto1978diffusion}
\begin{align}
\dfrac{\partial u}{\partial t} + u\dfrac{\partial u}{\partial x} + \dfrac{\partial^2 u}{\partial x^2} +\dfrac{\partial^4 u}{\partial x^4}=0 \quad \text{for } x\in[0,L], t\ge 0.
\end{align}
The KS equation exhibits nonlinear spatiotemporal chaos and rapid error growth
under autoregressive rollout, with trajectory separation eventually saturating
at a scale set by the diameter of its bounded chaotic attractor, making it a
standard benchmark for data-driven forecasting of dynamical systems~\cite{pathak2018model, loya2025structurepreservingneuralordinarydifferential, LINOT2023111838}.
The dataset consists of $4600$ temporal snapshots of the KS solution, with each state represented on a grid of $64$ spatial points. Let $u_t\in\mathbb R^{64}$ denote the state at time index $t$. The surrogate is trained to approximate the one-step evolution operator so that the learned model produces the one-step prediction
\begin{align}
\hat{u}_{t+1}=\hat{F}(u_t).
\end{align}
The first $4000$ snapshots are used as the training region, while the remaining snapshots are reserved for testing. Long-horizon predictions are generated autoregressively. Given an initial condition $u_t$, we set $\hat{u}_{t,0}=u_t$ and recursively compute the long-horizon predictions
\begin{align}
\hat{u}_{t,n+1}=\hat{F}(\hat{u}_{t,n}),\qquad n=0,\ldots,n_{\rm rollout}-1.
\end{align}
We denote the corresponding true rollout as $ u_{t,n}=F^n(u_t)=u_{t+n}$. At rollout horizon $n$, the componentwise prediction error is
\begin{align}
E_{t,n}=\left|u_{t,n}-\hat{u}_{t,n}\right|\in\mathbb{R}_+^{64},
\end{align}
where the absolute value is taken componentwise.

The surrogate model $\hat{F}$ is a fully connected neural network consisting of an encoder and a prediction head. The encoder maps the input state $u_t\in\mathbb{R}^{64}$ to a latent feature vector
$\phi(u_t)\in\mathbb{R}^p$,
where $p=16$ in the present experiment. The prediction head then maps $\phi(u_t)$ to the one-step surrogate prediction $\hat{F}(u_t)$. The latent representation $\phi$ from the final trained surrogate is held fixed and used to define the geometry underlying the anchor-based error estimator. This ensures that all signed displacement features are computed in a common latent coordinate system across calibration folds.

For KS equation, GeoQ is evaluated recursively along the surrogate rollout. Specifically, let $\hat{E}_{t,n}$
denote the estimated surrogate error at rollout step $n$, and let $\psi_{t,n}=\psi_{k_a,k}(\hat{u}_{t,n})$ denote the feature vector constructed from the fixed-representation anchor displacement and the local support-distance feature. The increment model $g_\eta:\mathbb{R}^{p+1}\rightarrow \mathbb{R}_+^{64}$ is trained using the componentwise pinball loss at quantile level $\tau=0.95$. At each rollout step, the error estimate is updated by
\begin{align}
\hat{E}_{t,n+1}
=\hat{E}_{t,n}+g_\eta(\psi_{t,n}),\qquad n=0,\ldots,n_{\rm rollout}-1.
\end{align}
Consequently, after \(n_{\rm rollout}\) rollout steps, we have
\begin{align}
\hat{E}_{t,n_{\rm rollout}}= \hat{E}_{t,0}+\sum_{n=0}^{n_{\rm rollout}-1}g_\eta(\psi_{t,n}).
\end{align}
This recursive construction propagates uncertainty along the same autoregressive trajectory used by the surrogate prediction.

To generate calibration data, we use a leave-one-cluster-out procedure based on $K$-means clustering with $K_f=5$. For each fold $\ell$, an auxiliary surrogate $\hat{F}^{(-\ell)}$ is trained using the data outside the held-out cluster and is evaluated autoregressively on initial conditions from the held-out cluster. The representation map $\phi$, however, is kept fixed across folds and is not retrained. The fold rollouts provide cross-fitted estimates of the componentwise rollout error increments used to train $g_\eta$. The number of local anchors is set to $k_a=50$ as the use of multiple anchors is expected to improve the stability of the anchor error baseline for chaotic autoregressive rollouts.

In this experiment, both the surrogate and GeoQ models are trained in the autoregressive rollout setting. The surrogate model is first trained on one-step pairs and then fine-tuned using a rollout loss obtained by recursively applying $\hat{F}$ over multiple time steps. During calibration, GeoQ rollout tuples are generated by rolling out the surrogate for $n_{\rm rollout}=300$ steps from each calibration initial condition. At each horizon, the componentwise error field is computed from the difference between the true trajectory and the surrogate trajectory, and $g_\eta$ is trained to estimate the corresponding rollout error increment. At test time, the trained GeoQ model is applied recursively for longer rollouts up to $600$ steps, allowing us to evaluate the extrapolation capacity of GeoQ beyond the calibration horizon.

\begin{figure}[t]
\centering

\subfloat[Spatially averaged error]{
  \adjustbox{valign=c}{%
    \includegraphics[width=0.30\textwidth]{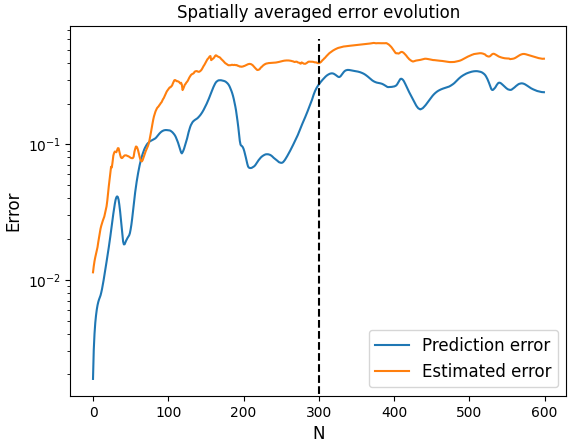}\label{ks_visu_a}
  }
}
\quad
\subfloat[Entire domain error]{
  \adjustbox{valign=c}{%
    \includegraphics[width=0.60\textwidth]{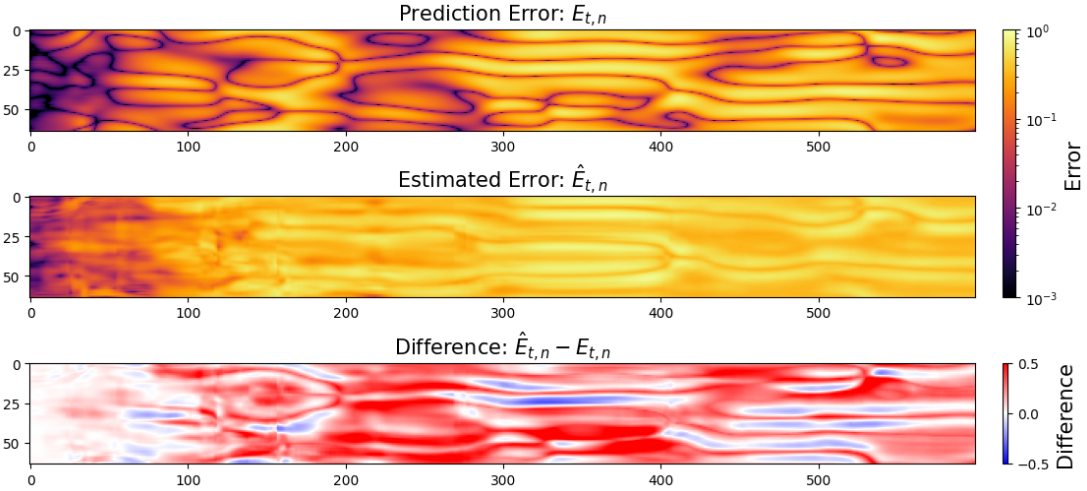}\label{ks_visu_b}
  }
}
\caption{\textit{KS equation:} Visualization of the surrogate's prediction errors and the estimated errors by the proposed GeoQ approach.}
\label{ks_visu}
\end{figure}

Figure~\ref{ks_visu} shows the spatially averaged error, the componentwise KS rollout error, and the corresponding GeoQ estimate over a $600$-step autoregressive forecast for a representative test initial condition. GeoQ is calibrated using rollout trajectories of length $n_{\rm rollout}=300$, indicated by the dashed vertical line in the spatially averaged error plot, and is then evaluated beyond this calibration horizon. The estimated error captures the rapid early growth and the subsequent saturation of the rollout error. In particular, the spatially averaged estimate remains close to or above the true mean error both within and beyond the calibration horizon, indicating that the recursive GeoQ model learns the dominant long-horizon error-growth trend. On the entire domain, however, the estimate is smoother than the true error field and does not fully resolve the narrow coherent error structures produced by chaotic phase drift. This leads to moderate pointwise correlation and some localized undercoverage, even though the average error magnitude is captured reasonably well. These results suggest that GeoQ provides a useful conservative estimate of long-horizon error growth, while fine-scale spatial localization remains challenging for chaotic autoregressive rollouts.

\subsection{WeatherBench dataset}

We consider the problem of learning a surrogate dynamical model for large-scale atmospheric evolution using coarse-resolution reanalysis data. The state variable
$X(t) \in \mathbb{R}^{H \times W \times C}$
represents a global field defined on a regular latitude–longitude grid (here $H=32$, $W=64$) with $C=3$ channels corresponding to geopotential height at $500$ hPa $(z_{500})$, temperature at $850$ hPa $(T_{850})$, and $2$ m temperature $(T_{2m})$.

Given a sequence of past states
$\{X(t-2), X(t-1), X(t)\}$
the goal is to predict a future trajectory over a horizon of $10$ days (that is, $40$ time steps at 6-hour resolution)
$\{X(t+1), X(t+2), \dots, X(t+40)\}$. The surrogate model therefore takes the following form:
\begin{align}
\hat{F} : \big(X(t-2), X(t-1), X(t)\big) \mapsto \{X(t+\tau)\}_{\tau=1}^{40}.
\end{align}

The surrogate model is implemented as a convolutional encoder–decoder (U-Net). The training is performed using a weighted mean-squared-error loss, with increasing weights for larger lead times to mitigate error accumulation.

The representation map $\phi(x)$ in the proposed GeoQ approach is chosen to be the identity function (i.e $\phi(x)=x$). To generate calibration data (with a leave-one-cluster-out procedure), we use K-means clustering procedure with $K_f=5$. The increment model $g_\eta:\mathbb{R}^{3\times32\times64\times3}\times\mathbb{R}\rightarrow\mathbb{R}^{40\times32\times64\times3}_+$ \bogus{$g_\eta: \mathbb{R}^{3\times32\times64\times3+1}\rightarrow \mathbb{R}^{40\times32\times64\times3}_+$} is a U-Net with multi-scale feature extraction. The scalar input is embedded through a small multilayer perceptron and broadcast spatially before being concatenated with the physical fields. Training is performed using a weighted pinball loss, with increasing weights for larger lead times to mitigate error accumulation.

We consider a subset of the WeatherBench dataset \citep{rasp2020weatherbench} for training and testing. The surrogate models are trained on $150$ consecutive days of data, and their predictive performance is evaluated over the subsequent $10$-day period following the final training day.

Figure \ref{WeatherBench_visu} illustrates the surrogate predictions and the corresponding estimated errors by GeoQ for the final forecast step (day $10$). Overall, the surrogate captures the large-scale atmospheric structures reasonably well across all variables, although localized discrepancies become visible in regions with stronger spatial variability. The proposed error estimator successfully identifies these high-error regions, producing spatial patterns that closely resemble the true prediction error fields. In particular, the estimated errors recover both the location and the relative intensity of the dominant error structures. However, the behavior varies somewhat across variables. For $z_{500}$ and $T_{850}$, the estimated error fields closely follow the magnitude and spatial localization of the true errors, indicating that the input-space geometry remains informative over the 10-day forecasting horizon. In contrast, the $T_{2m}$ field exhibits a more conservative uncertainty estimate, with broader regions of elevated predicted error. 

The strong agreement between the predicted and true error maps suggests that the proposed GeoQ framework captures meaningful relationships between input-space displacement, support structure, and local surrogate sensitivity. Despite being calibrated only from historical training data, the method generalizes effectively to temporally shifted forecasting conditions and remains responsive to dynamically difficult regions. These results demonstrate that the proposed approach can provide spatially resolved and physically consistent uncertainty estimates for medium-range surrogate weather forecasting.

\begin{figure}[H]
  \centering
  \subfloat[$z_{500}$]
{\includegraphics[width=\textwidth]{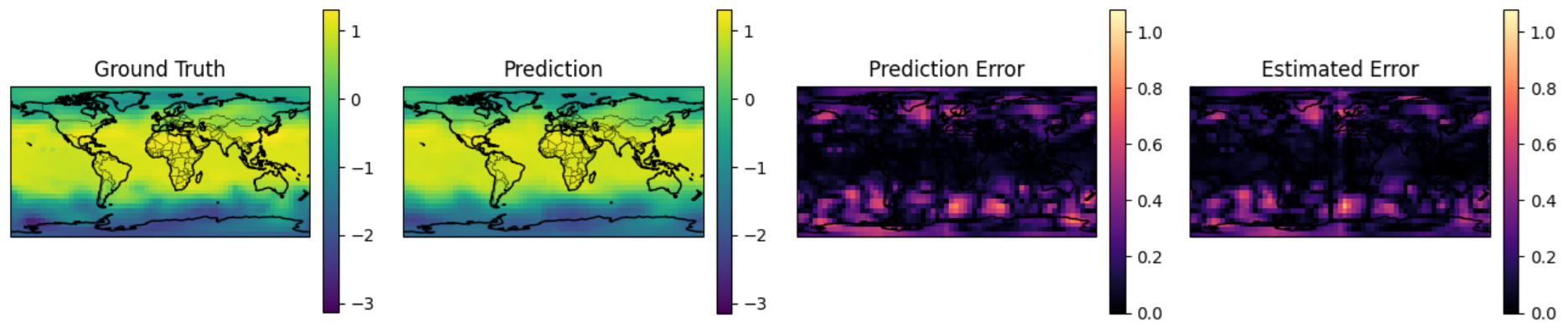}}
\quad
  \subfloat[$T_{850}$]{\includegraphics[width=\textwidth]{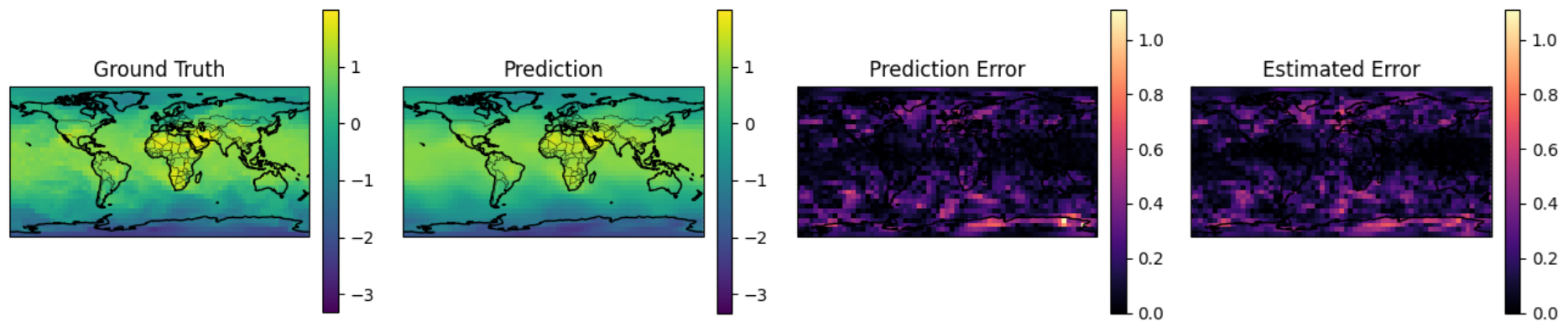}}
  \quad
\subfloat[$T_{2m}$]
{\includegraphics[width=\textwidth]{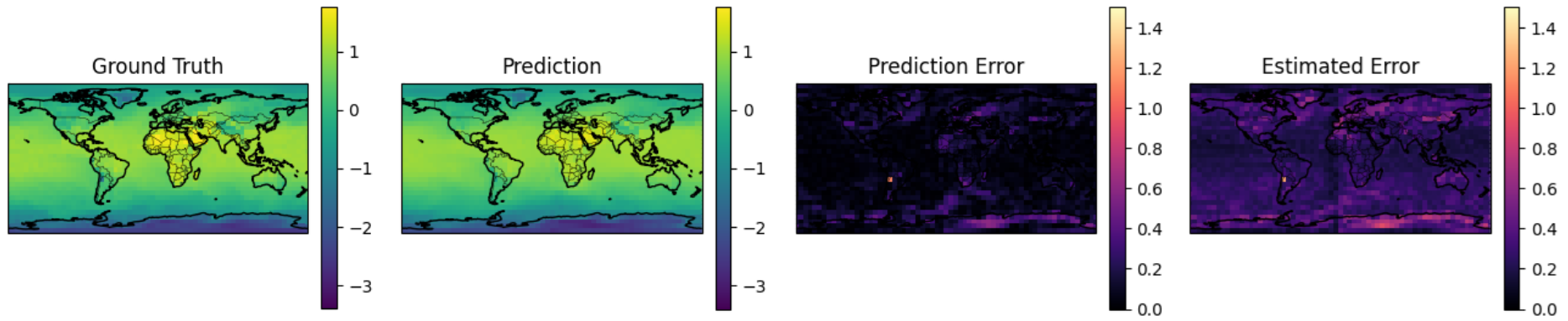}}
% \quad
  \caption{\textit{WeatherBench dataset: }Visualization at day $10$ of the surrogate prediction and the estimated error by the proposed GeoQ approach.}\label{WeatherBench_visu}
\end{figure}

\subsection{Richtmyer–Meshkov instability}
The Richtmyer–Meshkov instability (RMI) describes the growth of perturbations at an interface separating fluids of different densities after impulsive acceleration by a shock wave. In the present work, we consider a canonical two-dimensional planar Richtmyer-Meshkov instability in a shock-tube configuration with slab geometry, motivated by recent experimental studies \citep{schalles2024shock}. The flow evolves from an initially perturbed density interface separating gases of different densities, producing highly nonlinear roll-up structures, vortex interactions, and mixing layers over time. The RMI problem provides a particularly challenging benchmark for uncertainty estimation because the instability transitions from relatively smooth interface evolution at early times to strongly nonlinear vortex-dominated dynamics at later times. Consequently, uncertainty estimates must capture the global growth of prediction error and the spatial localization of dynamically unstable regions associated with interface roll-up and turbulent mixing \citep{zhai2018review}.

We evaluate the proposed GeoQ framework on the RMI problem using a dataset consisting of $1000$ simulations on a $240\times480$ grid. The surrogate model is trained to predict the temporal evolution, which includes 12 temporal snapshots, of the binary density field from the initial condition
\begin{align}
\hat{F} : u_0 \mapsto \{u_t\}_{t=1}^{12}.
\end{align}

The dataset consisting of $1000$ simulations is partitioned into $10$ folds using a $K$-means clustering procedure applied to the initial conditions. One fold is selected as the testing set, while the remaining nine folds are combined to form the training set. This strategy ensures that the testing data contained physically distinct initial conditions rather than random samples drawn from the same local distribution. As a baseline surrogate model, we employed a standard Fourier Neural Operator (FNO) trained directly on the binary density fields using a weighted binary cross-entropy and Dice loss.

The representation map $\phi(x)$ in the proposed error estimator approach is chosen to be the identity function (i.e $\phi(x)=x$). To generate calibration data (with a leave-one-cluster-out procedure), we use K-means clustering procedure with $K_f=10$. The increment model $g_\eta:\mathbb{R}^{240\times480}\times\mathbb{R} \rightarrow \mathbb{R}^{240\times480}_+$ \bogus{$g_\eta: \mathbb{R}^{240\times480+1}\rightarrow \mathbb{R}^{240\times480}_+$} is a U-Net-style convolutional neural network with FiLM-based conditioning layers, enabling the scalar nearest-neighbor metric to modulate intermediate feature representations throughout the network. Training the increment network presents a significant challenge because the target error fields are highly sparse. For most spatial locations, the surrogate prediction is already accurate, resulting in near-zero error throughout large portions of the domain. The dominant errors are concentrated near the material interface, particularly in regions of strong roll-up and filament formation. We employ a weighted loss formulation that minimizes the pinball loss (with the quantile level $\tau=0.95$) while simultaneously penalizing false-positive predictions in the background region.

\begin{figure}[t]
  \centering
%   \subfloat[$t=1$]
% {\includegraphics[width=\textwidth]{figures/paper_UQ_RMI_visu1.png}}
% \quad
  \subfloat[$t=6$]{\includegraphics[width=0.8\textwidth]{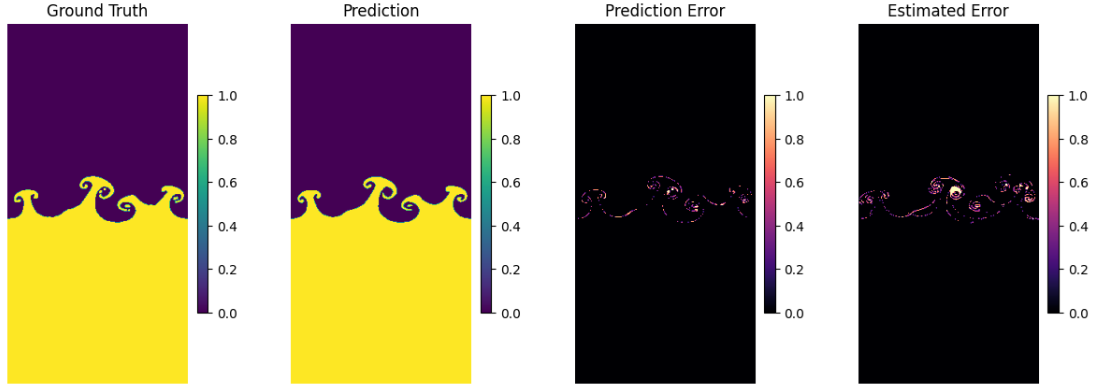}}
  \quad
\subfloat[$t=12$]
{\includegraphics[width=0.8\textwidth]{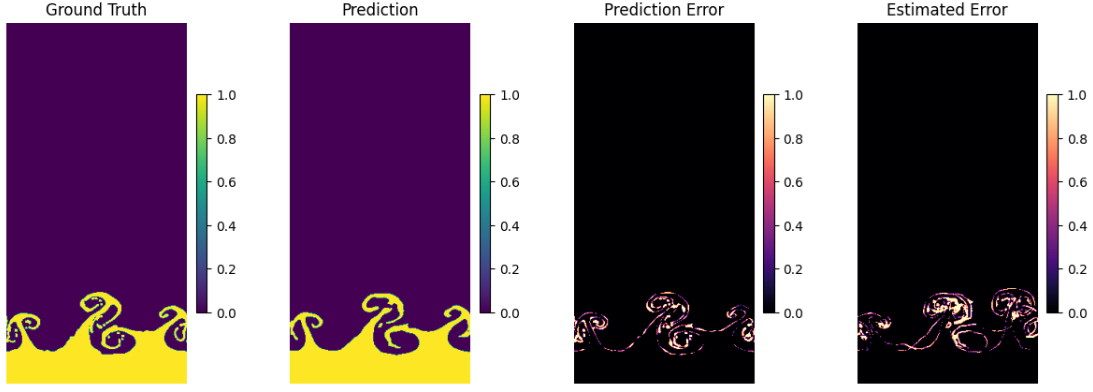}}
% \quad
  \caption{\textit{RMI problem:} Visualization of the surrogate prediction and the estimated error by the proposed GeoQ approach at different times for the median test case. The ground truth and the surrogate prediction are binary images for RMI evolution. The field represents the volume fraction of heavy gas, with $1$ denoting regions occupied by heavy gas and $0$ denoting regions with light gas.}\label{rmi_visu}
\end{figure}

Figure \ref{rmi_visu} presents representative results for the RMI problem at two different evolution times. At the earlier time step ($t=6$), the surrogate accurately captures the large-scale morphology of the instability, including the roll-up structures along the interface. The true prediction error remains localized near regions of interface deformation and vortex formation. The proposed GeoQ approach successfully identifies these localized error regions and reproduces the dominant spatial structures of the true error field. At the later time step ($t=12$), the flow becomes substantially more nonlinear due to the growth of small-scale vortical structures and increased interface complexity. Although the surrogate prediction still preserves the overall instability pattern, localized discrepancies become more pronounced around the rolled-up interfaces and thin filamentary regions. The estimated error field remains strongly correlated with the true error distribution and captures the increase in error concentration associated with the more complex dynamical regime.

These results indicate that the GeoQ framework is able to track the spatial evolution of surrogate uncertainty in a highly nonlinear dynamical system. In particular, the estimator remains sensitive to regions where instability growth and vortex interactions amplify prediction errors, suggesting that the learned geometric features successfully encode meaningful information about the evolving dynamics. Overall, the results demonstrate that the proposed method can provide spatially localized uncertainty estimates for long-horizon surrogate predictions of complex fluid instabilities.

\subsection{Quantitative results and safe-region analysis}\label{sec_quantitative}

We now present quantitative comparisons of the proposed GeoQ framework against several established uncertainty quantification approaches across the considered benchmark problems. Since the proposed method explicitly combines uncertainty estimation with a validity-aware safe-region mechanism, we report both global performance metrics and metrics restricted to the learned safe region. The global results (Table \ref{tab:global_results}) evaluate performance over the entire testing set, including both near-support and out-of-distribution samples. In contrast, the safe-region results (Table \ref{tab:safe_region_results}) evaluate only those samples whose geometric features remain sufficiently represented by the calibration data according to the feature-space density criterion introduced in Section \ref{sec:method}. This distinction is important because GeoQ provides error estimation with an explicit indication of whether the query lies in a calibration-supported region.

We note that the proposed GeoQ framework also contains several problem-dependent design choices, including the number of folds $K_f$ used in the cross-fitting procedure, the number of local anchors $k_a$, the number of neighbors $k$ used in the support feature, and the choice of the representation map $\phi$. These quantities influence the geometric structure of the calibration tuples and therefore affect the resulting uncertainty estimates. In practice, the optimal configuration depends on the dimensionality, dynamics, and distributional structure of the underlying problem. Performing a comprehensive ablation study across all considered benchmarks would therefore substantially increase the scope and computational cost of the present work. Instead, we provide representative ablation studies on the KS equation in \ref{append:ablation}, where we analyze the sensitivity of GeoQ to the principal hyperparameters and representation choices. The KS equation is selected because it exhibits nonlinear spatiotemporal dynamics while remaining computationally tractable for repeated calibration experiments.

\subsubsection{Global quantitative comparison}

We compare the proposed GeoQ framework against several established uncertainty quantification approaches, including Gaussian processes (GPs), MC dropout, deep ensembles, and split conformal prediction. MC dropout is used as a scalable practical approximation to Bayesian neural-network inference \citep{gal2016dropout}, where dropout at test time induces an approximate posterior predictive distribution. 

GeoQ directly predicts an upper estimate of the surrogate error through conditional quantile modeling. Split conformal prediction similarly produces explicit upper error bounds through calibration residual quantiles. In contrast, GPs, MC dropout, and deep ensembles naturally produce predictive standard deviations rather than direct error estimates. To enable a consistent comparison across methods, we convert these uncertainty estimates into approximate error bounds using $\widehat E(x)=1.96\sigma(x)$,
where $\sigma(x)$ denotes the predictive standard deviation. Under a Gaussian assumption, this corresponds to an approximate $95\%$ confidence interval. Although this quantity is not strictly equivalent to the conditional quantile estimate produced by GeoQ, it provides a common approximation for comparing uncertainty magnitudes across probabilistic surrogate modeling approaches. Additional implementation details for all baselines are provided in \ref{append_baseline}, while the surrogate model accuracy is reported in \ref{append_surrogate_accuracy}.

Table \ref{tab:global_results} compares the proposed GeoQ method against other approaches across all testing samples, including both near-support and out-of-distribution regions. We note that GP baselines are omitted for WeatherBench and RMI due to the high dimensionality of the input and output fields and the prohibitive scaling of exact GP inference. Besides that, correlation is omitted for conformal prediction when the conformal bound is a constant. For autoregressive baselines (in KS equation), the reported inference time includes all stochastic or ensemble rollout trajectories used to compute predictive standard deviations, whereas GeoQ and conformal prediction require only a single surrogate rollout plus a deterministic error-estimation step. Overall, the GeoQ framework achieves performance that is consistently competitive with existing methods while providing a geometry-aware and validity-aware formulation for uncertainty estimation.

\begin{table}[t]
\centering

\caption{Global performance comparison on the entire testing set, including both near-support and out-of-distribution samples. The proposed GeoQ framework is evaluated without restricting predictions to the learned safe region. Best results are shown in bold and second-best results are underlined.
}
\label{tab:global_results}
% \resizebox{\columnwidth}{!}{%
\begin{tabular}{llccccc}
\toprule
\textbf{Dataset} & \textbf{Method} & \textbf{Pinball} $\downarrow$ & \textbf{Correlation} $\uparrow$ & \textbf{Coverage} $\uparrow$ &
\textbf{Train. Time} $\downarrow$&
\textbf{Inf. Time} $\downarrow$ \\
\midrule

\multirow{3}{*}{Forrester eq.}
& GP     & \textbf{0.001} & 0.799 & \textbf{1.000}  & 20 s & 0.02 s \\
& MC Dropout     & 0.300 & 0.977 & 0.150  & 40 s  & 4 s \\
& Deep Ens.  & 0.303 & \underline{0.984} & 0.487  & 400 s & 4 s\\
& Conf. Pred. & 0.298  & - & 0.490 & 30 s & 0.6 s \\
& GeoQ (ours)     & \underline{0.117} & \textbf{0.989} & \underline{0.804}  & 500 s & 1.2 s \\

\midrule
\multirow{3}{*}{KS eq.}
& GP     & \underline{0.021} & \textbf{0.551} & 0.719 & 158 min & 9 min \\
& MC Dropout     & 0.035 & 0.108 & 0.706 & 9 min & 5 min\\
& Deep Ens.   & 0.039 &  0.265 &0.718 & 45 min & 5 min\\
& Conf. Pred.  & 0.023 & 0.372 & \underline{0.840} & 7 min & 2 min\\
& GeoQ (ours)  & \textbf{0.019} & \underline{0.490} & \textbf{0.868} & 65 min & 3.6 min\\

\midrule

\multirow{3}{*}{WeatherBench}
& MC Dropout     & 0.038 & 0.205 & 0.677    & 3 min & 20 s\\
& Deep Ens.  & 0.029 & \underline{0.612} & 0.681  & 30 min & 20 s\\
& Conf. Pred. & \underline{0.009} & 0.415 & \textbf{0.921}  & 2.5 min & 6 s \\
& GeoQ (ours)           & \textbf{0.006} & \textbf{0.763} & \underline{0.902} & 36 min & 1.2 s\\

\midrule
\multirow{3}{*}{RMI}
& MC Dropout     & 0.004 & 0.384 & 0.995  & 1.1 h & 3 min\\
& Deep Ens.  & \underline{0.002} & \textbf{0.645} & \underline{0.997}  & 11.0 h & 3 min\\
& Conf. Pred. & 0.004 & 0.377 & 0.990  & 1 h&  26 s\\
& GeoQ  (ours)    & \textbf{0.001} & \underline{0.501} & \textbf{1.000} & 12.5 h & 51 s\\

% \midrule
% \multirow{3}{*}{Reacting Flow}
% & MC Dropout     &  &  &    &  & \\
% & Deep Ens.  &  5.5$\times 10^{-3}$ & 0.657 & 0.910 &  227 h & 11 min \\
% & Conf. Pred. &  6.7$\times 10^{-3}$&  0.255&  0.947&   57 h & 5.5 min \\
% & GeoQ  (ours)    &  &  &  &  & \\

\bottomrule
\end{tabular}
% }
\end{table}

For the low-dimensional Forrester benchmark, Gaussian process regression achieves the best pinball loss and perfect coverage, which is expected due to the smooth structure and small dimensionality of the problem. Nevertheless, the GeoQ approach achieves the highest correlation between predicted and true errors while maintaining competitive coverage, indicating that the learned geometric representation captures meaningful local error variations. \textcolor{black}{This comparison should be interpreted together with the surrogate-error results in ~\ref{append_surrogate_accuracy}. For the Forrester benchmark, the GP baseline uses a different surrogate model and achieves a substantially smaller relative $\mathcal L^2$ error than the neural-network surrogates. Thus, the strong GP performance reflects both the suitability of Gaussian process regression for this low-dimensional smooth problem and the behavior of its uncertainty estimate. The neural-network baselines provide a more direct comparison to GeoQ in terms of base-surrogate class.}

For the KS equation, GeoQ achieves the lowest pinball loss and the highest coverage among all methods, indicating that it provides the most effective upper error estimate for the long-horizon autoregressive rollout task. Although GP obtains a slightly higher pointwise correlation, it substantially undercovers the true error, while GeoQ better balances calibration quality and spatial tracking. The moderate correlation of GeoQ reflects the difficulty of localizing fine-scale error structures in chaotic rollouts, where small phase shifts can strongly affect pointwise agreement. This result demonstrates \textcolor{black}{GeoQ's} capacity of providing reliable and computationally efficient uncertainty estimates in a challenging extrapolative and autoregressive regime.

% For the KS equation, GeoQ achieves the lowest pinball loss among all methods and the second-highest coverage, indicating that the anchor-averaged conditional quantile model provides sharp and well-calibrated upper error estimates for chaotic autoregressive rollout. Although GP regression attains the highest correlation and coverage, it is substantially more expensive at inference time. GeoQ therefore offers a favorable tradeoff between calibration quality and computational efficiency, achieving the best pinball loss while maintaining competitive coverage under a fair recursive uncertainty propagation setting.

On WeatherBench data, GeoQ achieves the best pinball loss, and correlation while maintaining coverage close to conformal prediction. This indicates that GeoQ provides both calibrated and spatially coherent error estimates for medium-range weather forecasting. The result suggests that input-space displacement and support information are effective for capturing localized atmospheric prediction errors.
% , and SSIM,

For \textcolor{black}{the} RMI problem, GeoQ obtains the best pinball loss and perfect coverage, demonstrating strong calibration for localized instability-driven errors. Deep ensembles achieve higher correlation, suggesting better spatial ranking of the error field. Nevertheless, GeoQ provides the most reliable upper error bounds, which is important for conservative uncertainty estimation in sparse interface-dominated error fields.
%and SSIM

%Across all benchmarks, the proposed GeoQ approach consistently achieves either the best or second-best performance for the majority of evaluation metrics. In particular, the method demonstrates a favorable tradeoff between calibration quality, structural consistency, and computational efficiency at inference time. These results support the effectiveness of GeoQ for uncertainty estimation in scientific surrogate modeling problems.
Across the selected benchmarks, GeoQ achieved competitive performance relative
to the considered baselines. GeoQ represents a methodologically distinct
approach to surrogate-error estimation, with different strengths and
limitations depending on the structure of the data, the available calibration
set, and the computational budget. In particular, its value lies in combining
query-dependent error estimation with an explicit measure of calibration
support, while requiring additional surrogate training during cross-fitted
calibration.

\subsubsection{Validity-aware performance}
Table \ref{tab:safe_region_results} reports the performance of the proposed GeoQ method restricted to the learned safe region. The safe region is defined using the feature-space kNN validity criterion introduced in Section \ref{sec:method}, where the safe threshold $\tau_{\mathrm{safe}}$ is chosen as the $95^{th}$ quantile of the calibration support scores. The results show that restricting predictions to samples lying inside the calibration-supported feature region improves the calibration quality of the estimated uncertainty relative to the global results reported in Table \ref{tab:global_results}. 

For the Forrester equation, the method achieves perfect coverage and nearly zero pinball loss on approximately $58\%$ of the testing samples, indicating that the learned geometry successfully identifies regions where the conditional quantile model remains highly reliable. For the KS equation, RMI, and WeatherBench datasets, the safe fraction remains high, while the corresponding safe-region metrics remain consistently competitive. For KS equation and RMI problem, the safe-region restriction modestly improves coverage while retaining a large fraction of samples. For WeatherBench, all test samples remain within the learned safe region, so the safe-region metrics coincide with the global metrics.

% For the more challenging KS and WeatherBench cases, the safe-region coverage remains substantially improved relative to the corresponding global uncertainty estimates.

These results demonstrate that GeoQ can additionally provide
a validity-aware indication of whether the learned geometric calibration is supported by the available data.
This validity-aware behavior is particularly important in extrapolative and distribution-shifted settings, where global uncertainty estimates alone may become unreliable.

% These results highlight an important property of the proposed GeoQ framework: rather than assuming reliability everywhere, the method explicitly identifies regions where the learned geometric calibration remains supported by the available data. This validity-aware behavior is particularly important in extrapolative and distribution-shifted settings, where global uncertainty estimates alone may become unreliable.

\begin{table}[t]
\centering
\caption{
Validity-aware performance of the proposed method inside the learned safe region.
Coverage$_{\mathrm{safe}}$ and Pinball$_{\mathrm{safe}}$ are computed only on samples satisfying
$S(\psi)\le\tau_{\mathrm{safe}}$.
Higher Coverage and Safe Fraction are better, lower Pinball is better.
}
\label{tab:safe_region_results}
\renewcommand{\arraystretch}{1.15}
\begin{tabular}{lcccc}
\hline
\textbf{Problem}
& \textbf{Safe Fraction} 
& \textbf{Pinball}$_{\mathrm{safe}}$ 
& \textbf{Correlation}$_{\mathrm{safe}}$ 
& \textbf{Coverage}$_{\mathrm{safe}}$ 
\\
\hline

Forrester equation
& 0.582
& 0.000
& 0.978
& 1.000
\\

KS equation
& 0.900
& 0.017
& 0.495
& 0.891
\\

% RMI (Mode OOD)
% & 1.000
% & 0.000
% & 0.452
% & 1.000
% \\

WeatherBench
& 1.000
& 0.006
& 0.763
& 0.902
\\

RMI 
& 0.959
& 0.001
& 0.502
& 1.000
\\

\hline
\end{tabular}
\end{table}

\section{Conclusion and Future work}\label{sec:conclusion}

In this work, we introduced GeoQ, a geometry-aware conditional quantile framework for error estimation in scientific surrogate models. The proposed method estimates the surrogate error at a query point by combining an anchor-averaged calibration error with a learned nonnegative correction based on geometric displacement and local support information. By training this correction as an upper conditional quantile of cross-fitted error increments, GeoQ provides spatially and temporally resolved estimates of surrogate error while avoiding overly conservative global worst-case bounds. In addition, the feature-space kNN safe-region criterion provides an explicit mechanism for identifying regions where the learned calibration remains supported by the available data.

Across the Forrester function, Kuramoto-Sivashinsky equation, WeatherBench dataset, and Richtmyer-Meshkov instability problem, GeoQ achieved competitive, and in several cases superior, calibration-oriented performance relative to Gaussian processes, MC dropout, deep ensembles, and split conformal prediction. The results show that GeoQ is particularly effective in terms of pinball loss, indicating strong calibration of upper-tail error estimates. For WeatherBench and RMI, GeoQ also produced spatially coherent error fields and reliable coverage, demonstrating its ability to estimate localized uncertainty in high-dimensional scientific prediction problems. For the KS equation, GeoQ achieved the lowest pinball loss and highest coverage among the compared methods, showing that the recursive GeoQ formulation provides an effective error-bound estimate for chaotic autoregressive rollouts. Although precise pointwise localization of fine-scale error structures remains challenging in this setting, GeoQ captures the dominant error-growth behavior and provides the best overall calibration-quality tradeoff.

Several questions remain for future work. The choice of representation map, calibration-fold construction, and number of neighbors used in the support score is problem dependent, and adaptive or data-driven procedures for selecting these components could improve robustness across applications. The cross-fitting procedure also requires training an auxiliary surrogate for each fold, which can dominate the calibration cost, motivating approaches that construct approximately out-of-sample error targets without training $K_f$ complete auxiliary surrogates. The current safe-region criterion is empirical, and stronger connections between feature-space support and conditional coverage would clarify when the resulting error estimates can be trusted. Recursive error propagation could also incorporate learned amplification factors or stability-aware dynamics to better represent long-horizon error growth in chaotic systems. Finally, GeoQ may be useful within inverse problems and optimization, where query-dependent error estimates could inform trust regions, penalties, or acceptance criteria for candidate solutions that lie away from observed training samples but remain within the calibration-supported region. The estimated error and support score could also guide adaptive data acquisition by identifying where additional high-fidelity simulations or experiments would most improve a surrogate or digital twin.

\section{Acknowledgements}
This work was supported by the Advanced Simulation and Computing (ASC) Program at Los Alamos National Laboratory. Los Alamos National Laboratory Report LA-UR-26-26528.

%noise

% \bibliographystyle{unsrtnat} 
\bibliography{refs} 
\appendix

\section{Model architectures and training configurations}
\label{append:model_architecture}

All surrogate and uncertainty models are trained using the Adam optimizer with learning rate \(10^{-3}\). Unless otherwise stated, training is performed using mini-batch stochastic optimization with early stopping based on validation loss. The surrogate architectures employed for the considered benchmark problems are summarized below.

\begin{itemize}

\item \textbf{Forrester equation:}
The surrogate model consists of a fully connected feedforward neural network with hidden-layer widths
$(64,\,64,\,16,\,32)$,
where the \(16\)-dimensional layer defines the latent representation map \(\phi\). ReLU activations are used for all hidden layers, followed by a linear output layer.

\item \textbf{KS equation:}
The surrogate model is implemented as a fully connected neural network with hidden-layer widths
$(128,\,128,\,16,\,128,\,128)$.
The intermediate \(16\)-dimensional latent layer is used as the representation map \(\phi\). ReLU activations are employed throughout the network.

\item \textbf{WeatherBench:}
For the WeatherBench forecasting problem, we use a convolutional encoder-decoder U-Net that maps a three-step atmospheric history to a 40-step forecast trajectory. The encoder consists of three downsampling levels with convolutional blocks using $32$, $64$, and $128$ filters, respectively. Each block contains two $3\times 3$ convolutional layers with GELU activations followed by dropout with rate $0.1$. A bottleneck block with $256$ filters is used at the coarsest resolution. The decoder mirrors the encoder using nearest-neighbor upsampling, skip connections from the corresponding encoder levels, and convolutional blocks with decreasing filter sizes. A final $1\times 1$ convolution produces $40\times 3$ output channels, which are reshaped and permuted to yield the multi-step forecast field.

\item \textbf{RMI:}
For the Richtmyer-Meshkov instability problem, we use a two-dimensional Fourier Neural Operator (FNO) that maps an initial binary density field of shape $240\times 480\times 1$ to a sequence of $12$ future density fields. The input field is first resized to an internal resolution of $256\times 512$, and normalized spatial coordinate channels are appended to provide positional information. The augmented input is lifted to a width of $64$ channels using a $1\times 1$ convolution. The main network consists of four FNO blocks, each combining a spectral convolution over the lowest $24\times 24$ Fourier modes with a local $1\times 1$ convolutional residual branch, followed by a GELU activation. After the Fourier layers, the representation is projected through a $1\times 1$ convolution with $128$ channels and a final $1\times 1$ convolution with sigmoid activation to produce the $12$ probability output fields. The output is then resized back to the original spatial resolution $240\times 480$.

\end{itemize}

For the increment model $g_\eta$ in GeoQ, we employ architectures similar to those used for the corresponding surrogate models, while modifying the input layers so that the networks take the geometric feature vector $\psi=(\Delta z_{k_a},\Delta d _k)$ as input. In particular, the latent displacement $\Delta z_{k_a}$ is treated as the primary geometric feature, while the scalar support score $\Delta d _k$ is incorporated through feature concatenation or conditioning layers depending on the underlying architecture.

All experiments are performed using a single NVIDIA GeForce RTX 2080 Ti GPU.

% The representation map $\phi$ is realized using two hidden layers with $64$ ReLU units each, followed by a linear latent layer. The surrogate prediction is then obtained through a regression head composed of one hidden layer with $32$ ReLU units and a final linear output layer:
% $$
% x \xrightarrow{\phi} z\in\mathbb{R}^{p}
% \xrightarrow{h} \hat{u}(x)
% $$
% The surrogate is trained using mean-squared error loss with the Adam optimizer, learning rate $10^{-3}$, batch size $32$, and $250$ training epochs.

\section{Ablation studies on KS equation}\label{append:ablation}

This appendix presents representative ablation studies for the proposed GeoQ framework using the KS equation benchmark. The objective of these experiments is to examine the sensitivity of the method to several key design choices, including the representation map $\phi$, the geometric feature construction $\psi$, the number of local anchors $k_a$, and the number of folds employed in the cross-fitting procedure $K_f$. The KS equation is selected for these studies because it exhibits nonlinear spatiotemporal dynamics while remaining computationally tractable for repeated calibration experiments. Since the optimal configuration of these components is generally problem-dependent, we restrict the detailed sensitivity analysis to this representative benchmark.

\begin{table}[H]
\centering

\caption{Comparison of different choices in the proposed approach on KS equation.
}

\label{tab:ks_results_append}
\begin{tabular}{lccc}
\toprule
\textbf{Choice} & \textbf{Pinball} & \textbf{Correlation} & \textbf{Coverage}\\

\midrule

$\phi=\operatorname{Id}_{\mathcal X}$            &  0.024 & 0.412 &    0.856\\
$\phi=\widehat F_{\mathrm{encoder}}$           &  \textbf{0.019} & \textbf{0.490} &   \textbf{0.868} \\

\midrule

$\psi=[\Delta z_{k_a}, \Delta d_k]$            &  \textbf{0.019} & \textbf{0.490} &   \textbf{0.868} \\
$\psi=[|\Delta z_{k_a}|, \Delta d_k]$            & 0.021  & 0.483 & 0.835   \\
$\psi=[||\Delta z_{k_a}||_2, \Delta d_k]$           &  0.020 & 0.481 & \textbf{0.868}   \\
$\psi=\Delta z_{k_a}$      & 0.021       & 0.478 & 0.832    \\
\midrule

$k_a=1$            & \textbf{0.017} & \textbf{0.505} &  \textbf{0.901}  \\
$k_a=10$           & 0.019 & 0.473 & 0.881   \\
$k_a=50$          &  0.019 & 0.490 &   0.868 \\
$k_a=100$          & 0.018 & 0.491  & 0.875   \\

\midrule

$K_f=2$            & 0.057& 0.243 & 0.549\\
$K_f=5$         &  \textbf{0.019} & \textbf{0.490} &   \textbf{0.868} \\
$K_f=10$           & 0.023 & 0.457 & 0.834   \\
$K_f=20$           & 0.024 & 0.444 &  0.812   \\

\bottomrule
\end{tabular}
\end{table}

The ablation results in Table \ref{tab:ks_results_append} show that the proposed GeoQ configuration is sensitive to the representation map, feature construction, anchor size, and number of calibration folds. In each block, one design choice is varied while the remaining settings are fixed to the following default configuration: $\phi=\widehat F_{\mathrm{encoder}}$, $\psi=(\Delta z_{k_a},\Delta d_k)$, $k_a=50$, and $K_f=5$.

The choice of representation map has a clear effect on performance. Using the surrogate encoder representation $F_{\mathrm{encoder}}$ improves all metrics compared with the identity map. This indicates that the learned latent representation provides a more informative geometry for anchor selection and error-increment modeling than the raw input coordinates.

The feature ablation shows that including the signed representation displacement $\Delta z_{k_a}$ together with $\Delta d_k$ gives the best overall balance between pinball loss, correlation, and coverage. We note that all signed displacement features are computed in a fixed shared representation space. Replacing $\Delta z_{k_a}$ with $|\Delta z_{k_a}|$ degrades all metrics, suggesting that directional information in representation space is useful for predicting error growth. Using only $||\Delta z_{k_a}||_2$ and $\Delta d_k$ yields slightly higher coverage but substantially lower correlation, indicating that scalar displacement magnitude alone loses important directional information. Removing $\Delta d_k$ also reduces performance, confirming that local support information contributes to calibration.

The anchor-neighborhood ablation shows that $k_a=1$ gives the strongest performance among the tested values. We see that increasing $k_a$ worsens performance, suggesting that too large an anchor size can oversmooth local geometry.

% The anchor-neighborhood ablation shows that $k_a=50$ gives the strongest performance among the tested values. Compared with $k_a=1$, using multiple anchors substantially improves pinball loss, correlation, and coverage. This supports the interpretation that anchor averaging stabilizes the local error baseline in chaotic autoregressive rollouts, where a single nearest anchor may have an unrepresentative error. However, increasing to $k_a=100$ slightly worsens performance, suggesting that too large an anchor neighborhood can oversmooth local geometry.

Finally, the fold ablation indicates that $K_f=5$ provides the best tradeoff between calibration diversity and surrogate fidelity. Too few folds reduce the amount of stable training data for each cross-fitted surrogate, while too many folds produce held-out samples that are less geometrically separated from the training set, weakening the pseudo-OOD calibration signal. Overall, these ablations support the default KS configuration and show that GeoQ benefits from learned representations, directional displacement features, local support information, and moderate anchor averaging.

\section{Baseline uncertainty methods and implementation details}\label{append_baseline}

In this work, we compare GeoQ with Gaussian process regression, MC dropout, deep ensembles, and split conformal prediction. Below is the description for the implementation of each baseline method:
\begin{itemize}
    \item \textbf{Gaussian Process:} For the low-dimensional benchmarks where Gaussian process regression is computationally feasible, we use a Gaussian process with a constant kernel multiplied by a radial basis function kernel and an additive white-noise kernel. Inputs are standardized before fitting, and outputs are normalized during GP training. For vector-valued outputs, we use an independent multi-output formulation in which a separate GP is fit to each output coordinate. The GP predictive standard deviation is converted into an approximate error bound using $\widehat E(x)=1.96\sigma(x)$. This provides an input-dependent uncertainty estimate that naturally increases away from the training data under the chosen kernel geometry, but its scalability is limited for high-dimensional field-valued problems.
    \item \textbf{MC dropout:} For MC dropout, we modify the surrogate architecture by inserting dropout layers with dropout rate $0.1$. During inference, dropout remains active and the same input is evaluated multiple times to generate a Monte Carlo sample of predictions. The predictive standard deviation across these stochastic forward passes is then used as the uncertainty estimate and converted into an approximate error bound using $\widehat E(x)=1.96\sigma(x)$. As with deep ensembles, for field-valued outputs this procedure produces a coordinatewise uncertainty field. Unlike GeoQ, MC dropout does not explicitly use an anchor representation or support-based validity criterion. Uncertainty arises from stochastic variation in the network predictions.
    \item \textbf{Deep ensembles:} For deep ensembles, we train multiple independently initialized surrogate models and use the empirical variability of their predictions as the uncertainty estimate. To make the computational cost comparable to GeoQ, we use the same number of ensemble members as the number of calibration folds (denoted $K_f$ in section \ref{sec:method}). Each ensemble member uses the same surrogate architecture and training configuration, but is trained independently with a fixed random seed protocol. At inference time, the ensemble mean is used as the surrogate prediction, and the ensemble standard deviation is converted into an approximate error bound using $\widehat E(x)=1.96\sigma(x)$. For spatially distributed outputs, the standard deviation is computed coordinatewise, yielding an uncertainty field with the same shape as the surrogate output.
    \item \textbf{Conformal prediction:} For conformal prediction, we use split conformal calibration on pointwise absolute residuals. The resulting bound is a scalar for scalar-output problems and a coordinatewise residual quantile field for spatially distributed outputs. Thus, the conformal baseline provides a fixed, input-independent error bound with the same output shape as the surrogate prediction.
\end{itemize}

For MC dropout, dropout is kept active during inference and multiple stochastic rollout trajectories are generated. For deep ensembles, each ensemble member is rolled out independently. For GP, predictive uncertainty is obtained by Monte Carlo propagation of GP samples through the autoregressive rollout. The reported inference times include all trajectories used to compute predictive standard deviations. In contrast, GeoQ evaluates a deterministic surrogate rollout followed by a deterministic error-estimation model, and conformal prediction applies a fixed calibrated residual bound.

\section{Surrogate model accuracy}\label{append_surrogate_accuracy}
To ensure that the uncertainty comparisons are not primarily driven by differences in surrogate accuracy, we report in Table \ref{tab:surrogate_error_append} the prediction errors of the surrogate models used by each uncertainty method. For GeoQ, MC dropout, deep ensembles, and conformal prediction, the underlying surrogate architectures are essentially the same, except that MC dropout includes dropout layers. Their relative $L^2$ errors are broadly comparable across the high-dimensional benchmarks, particularly for KS, WeatherBench, and RMI. This supports the interpretation that differences in pinball loss, coverage, and correlation mainly reflect differences in the uncertainty-estimation procedure rather than substantially different base surrogate quality. 

\begin{table}[H]
\centering

\caption{Comparison of the surrogate model accuracy (in terms of relative $L^2$ error) in different approaches.
}

\label{tab:surrogate_error_append}
\begin{tabular}{lcccc}
\toprule \textbf{Method} &\textbf{Forrester eq.} & \textbf{KS eq.} & \textbf{WeatherBench}
& \textbf{RMI}\\

\midrule
GP & 0.093 & 0.429 &- & -\\
MC Dropout & 0.879 & 0.522 & 0.176 & 0.082 \\
Deep Ens. & 0.884 & 0.446 & 0.227 & 0.082 \\
Conf. Pred. & 0.945 & 0.492 & 0.172 & 0.087 \\
GeoQ & 0.868 & 0.484 & 0.168 & 0.083 \\
\bottomrule
\end{tabular}
\end{table}

The GP baseline uses its own surrogate model, so its relative error is not directly tied to the neural surrogate architecture. This explains why GP performs much better on the low-dimensional Forrester problem, where Gaussian process regression is well suited and achieves a much smaller relative error. However, for the KS equation, the GP surrogate error is comparable to the neural surrogates, while GP is not reported for WeatherBench and RMI due to scalability limitations. Overall, the table confirms that GeoQ is evaluated on surrogate predictions with accuracy similar to the neural-network baselines, making the uncertainty comparison more meaningful.

\end{document}